\documentclass{article}
\usepackage{geometry}
\usepackage{graphicx}
\usepackage{subcaption}
\usepackage{booktabs}
\usepackage{amsthm,amssymb,url,amsmath}
\usepackage{mathtools}
\usepackage{todonotes}

\usepackage[colorlinks= true, linkcolor = black, citecolor = blue]{hyperref}

\newtheorem{theorem}{Theorem}[section]

\newtheorem*{claim*}{Claim}

\newtheorem{corollary}[theorem]{Corollary}
\newtheorem{property}[theorem]{Property}

\newtheorem{definition}[theorem]{Definition}

\newtheorem{lemma}[theorem]{Lemma}

\newtheorem{remark}[theorem]{Remark}

\title{Stochastic strategies for patrolling a terrain with a synchronized multi-robot system}

\author{Luis E. Caraballo$^1$ \and
      Jos\'e M. D\'iaz-B\'a\~nez$^{2,^*}$ \and
      Ruy Fabila-Monroy$^{3}$  \and
      Carlos Hidalgo-Toscano$^{3}$ 
\thanks{$^{1}$Departamento de Matem\'atica Aplicada II, Universidad de Sevilla, Spain. Supported by Spanish Government under the grant agreement FPU14/04705. He is partially supported by project GALGO (Spanish Ministry of Economy and Competitiveness, MTM2016-76272-R AEI/FEDER,UE). {\tt\small lcaraballo@us.es}}%
\thanks{$^{2}$Departamento de Matem\'atica Aplicada II, Universidad de Sevilla, Spain. Partially supported by project GALGO (Spanish Ministry of Economy and Competitiveness, MTM2016-76272-R AEI/FEDER,UE). {\tt\small dbanez@us.es}}%
\thanks{$^{3}$Departamento de Matem\'aticas, CINVESTAV, Mexico City, Mexico.
{\tt\small ruyfabila@math.cinvestav.edu.mx, cmhidalgo@math.cinvestav.mx}}%
\thanks{$^*$ Corresponding author}
      }

\begin{document}
\maketitle

\begin{abstract}
A group of cooperative aerial robots
can be deployed to efficiently patrol a terrain, in which each robot flies around an assigned area and shares information with the neighbors periodically in order to protect or supervise it. To ensure robustness, previous works on these synchronized systems propose sending a robot to the neighboring area in case it detects a failure. In order to deal with unpredictability and to improve on the efficiency in the deterministic patrolling scheme, this paper proposes 
random strategies to cover the areas distributed among the agents. First, a theoretical study of the stochastic process is addressed in this paper for two metrics: 
the \emph{idle time}, the expected time between two consecutive observations of any point of the terrain and the \emph{isolation time}, the expected time that a robot is without communication with any other robot. After that,
the random strategies are experimentally compared with the deterministic strategy adding another metric: the \emph{broadcast time}, the expected time elapsed from the moment a robot emits a message until it is received by all the other robots of the team. The simulations show that theoretical results are in good
agreement with the simulations and the random strategies outperform the behavior obtained with the deterministic protocol proposed in the literature.
\end{abstract}

{\bf Keywords:} Multi-agent systems; Random walks; Aerial surveillance; Synchronization

\begin{figure}[!b]
\begin{minipage}[c]{0.15\columnwidth}
\includegraphics[height=2.5em]{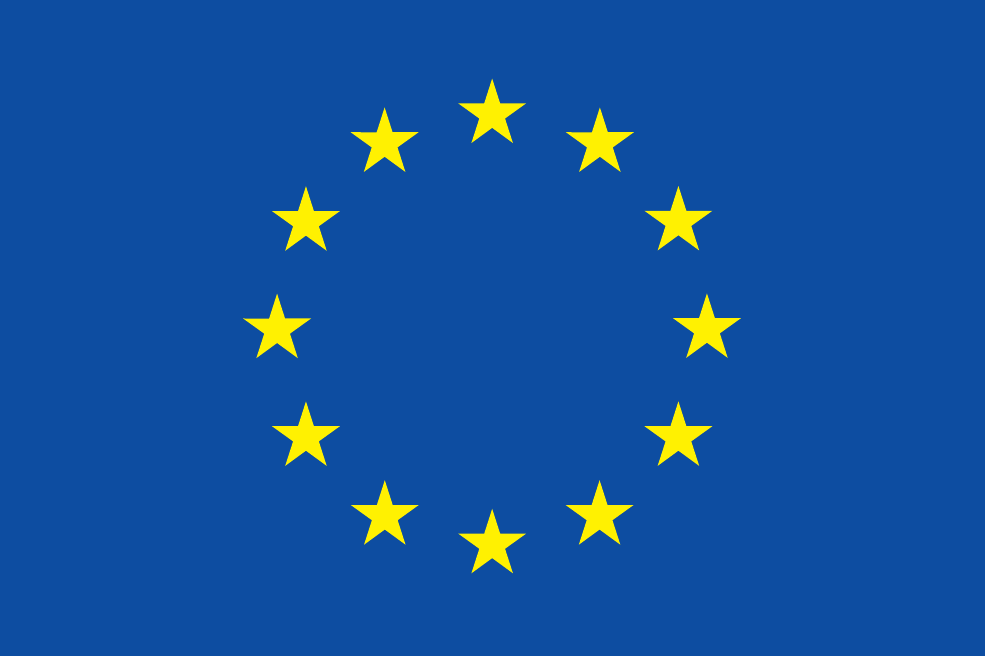}
\end{minipage}\quad
\begin{minipage}[l]{0.80\columnwidth}
\footnotesize This project has received funding from the European Union's Horizon 2020 research and innovation programme under the Marie Sk\l{}odowska-Curie grant agreement No 734922 and the Spanish Ministry of Economy and Competitiveness (GALGO, MTM2016-76272-R AEI/FEDER,UE)..
\end{minipage}
\end{figure}

\section{Introduction}

Unmanned aerial vehicles (UAVs), best known as drones, are  an emerging technology with
significant market potential and bring new challenges to the optimization field. The cooperation of multiple UAVs performing joint missions is being applied to many areas such as surveying, precision agriculture, search and rescue, monitoring in dangerous scenarios, exploration and mapping, etc. 
There exist interesting surveys on the topic~\cite{hayat2016survey,yanmaz2018drone,otto2018optimization}.

The coordination of a team of autonomous vehicles allows it to achieve
missions that no individual autonomous vehicle can accomplish on its own. For example, the team members can exchange information and collaborate to efficiently protect a big area from intruders \cite{alpern2019optimizing}. Each UAV is equipped with 
a camera and takes snapshots of the ground area. The
aim of the team is to
somehow cover this area using several snapshots. Several categories of UAVs (according to the type of wings, size or their
communication capabilities) have been considered depending on the application. 
Small quadrotors are of special interest, which take off and land vertically with many potential applications ranging from mapping to supporting
\cite{gupte2012survey,yanmaz2017communication}. 

A usual strategy in multiagent patrolling is to partition the terrain and to assign an area to each UAV \cite{acevedo2014one,ahmadi2006multi}. Assume that the partition is given and consider the following scenario. There is a set of $n$ aerial robots, each robot traveling on a fixed closed trajectory while
performing a prescribed task. Moreover, each robot needs to communicate periodically with
the other robots 
and the communication range is limited. Thus, in order
for two robots to communicate, they must be in close proximity to each other. 
 D\'iaz-B\'a\~nez et al.~\cite{diaz2017} present a framework to survey a terrain in this scenario. 
In their model, each trajectory is a Jordan curve with either a clockwise or counter clockwise
orientation. Each robot moves at constant speed and it takes a unit of time to complete
a tour of its assigned trajectory.
A pair of these trajectories may intersect but do not cross.
Each point of intersection, provides an opportunity for the corresponding robots to communicate. We say that
a communication link (bridge) exists between two trajectories if they intersect at some point. 
This model was further studied by Bereg et al.~\cite{bereg2018computing}, where the authors define a graph $G$ of \emph{potential links}.
The vertex set of this graph is the set of trajectories and two of them are
adjacent if they share a communication link. The graph is assumed to be connected. 
A \emph{synchronization schedule} is a tuple $(\alpha,\delta)$, where $\alpha$ is a tuple whose entries are the initial
positions of the robots, and $\delta$ is a tuple whose entries are the directions in which the robot traverse
their assigned trajectory (clockwise or counterclockwise). A synchronization schedule defines a subgraph $H(\alpha,\delta)$ of $G$, $H$ for short, in a natural way.
$H$ has the same vertices as $G$ and an edge of $G$ is in $H$ if the corresponding robots meet at the corresponding communication link.
$H$ is called the \emph{communication graph}. Examples of communication graphs are shown in Figure 
\ref{fig:synchro_sample}. 
The problem addressed by D\'iaz-B\'a\~nez et al.~$\cite{diaz2017}$ is to find a synchronization schedule so that $H$ is connected and 
with maximum number of edges. A set of trajectories with a synchronization schedule such that $H$ is connected is called
a \emph{synchronized communication system (SCS)}. For an illustration, see Figure~\ref{fig:synchro_sample}
and video at youtube\footnote{\mbox{{\url{https://www.youtube.com/watch?v=T0V6tO80HOI}}}}. D\'iaz-B\'a\~nez et al.~$\cite{diaz2017}$ also give necessary and sufficient conditions on when this synchronization is possible. Note that although not
every pair of robots can communicate directly in this framework, a robot may relay a message to a 
robot with which it does not have direct communication.

\begin{figure}[ht]
\centering
\begin{subfigure}{.4\columnwidth}
\centering
\includegraphics[width=.8\columnwidth]{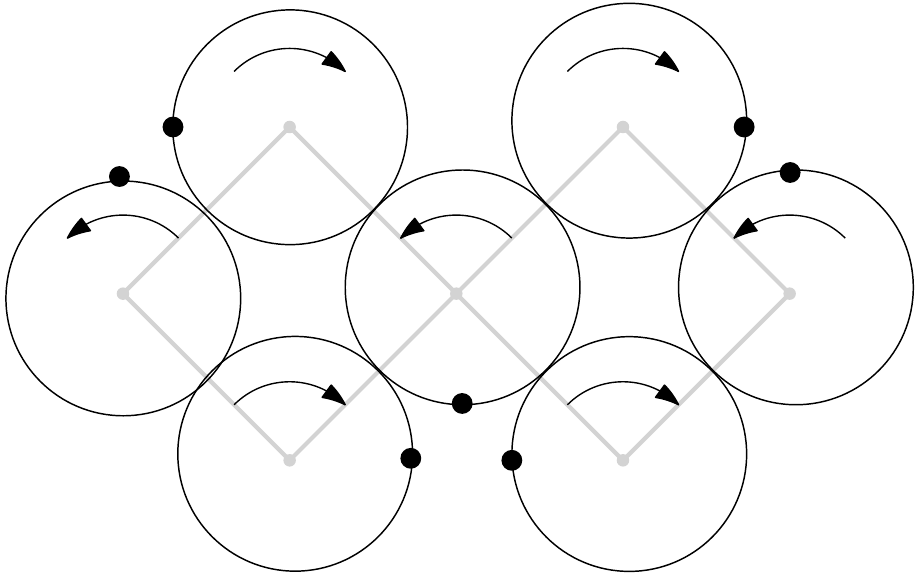}
\caption{}
\end{subfigure}%
\begin{subfigure}{.55\columnwidth}
\centering
\includegraphics[width=.5\columnwidth, page=1]{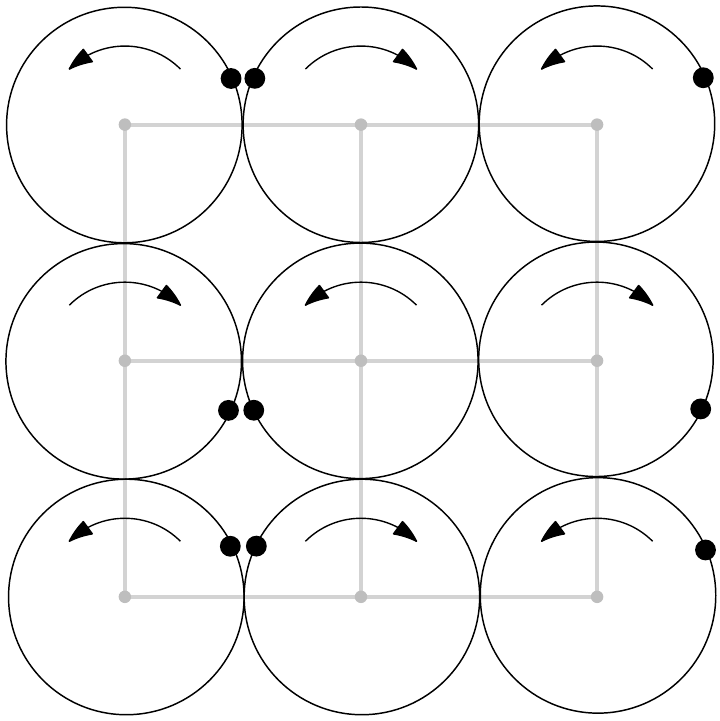}
\caption{}
\end{subfigure}%
\caption{Examples of synchronized communication systems (SCSs) with circular trajectories. The communication graph $H$ is given by light gray lines. The arrows indicate the travel directions in every circle and the solid black points indicate the starting positions.}
\label{fig:synchro_sample}
\end{figure}

Note also that the synchronization allows for a robot
to detect the failure of a neighboring robot. If a robot $u$ in a trajectory $A$ arrives at the communication link between $A$ and another trajectory $B$, and it detects that there is no robot in $B$, then $u$ assumes that the robot in $B$ is not longer functional. A possible strategy would be for $u$ to pass to $B$ and to take over the task of the missing robot. This move is called a \emph{shifting operation} by D\'iaz-B\'a\~nez et al.~$\cite{diaz2017}$.
A notable property of the framework to apply a shifting operation is that neighboring trajectories have opposite travel directions assigned (clockwise and counterclockwise) and a robot can change to a neighboring trajectory with ease. Clearly, if the trajectories are traveled in the same direction, the shifting operation requires a large turn angle of the aerial robot. 
In this paper, the protocol presented by D\'iaz-B\'a\~nez et al.~$\cite{diaz2017}$ is named the \emph{ deterministic strategy}, in which a robot performs a shifting operation every time it detects the absence of a neighbor (at the communication link). For example, in Figure \ref{fig:determ_prob}(b) the shifting is performed at all the communication links.

Consider an SCS with $n$ trajectories and $k<n$ available robots following the deterministic strategy.  Although the surviving robots perform shifting operations, it may occur that some robot always fails to meet any other robot at the communication links; that is, it is \emph{isolated}.
Note that an isolated robot always shift trajectory at a communication link. See examples in Figure~\ref{fig:determ_prob}(b) and Figure~\ref{fig:determ_prob}(c).
Moreover, some segments (pieces) of the trajectories may never be visited by a robot; i.e., some parts of the terrain are uncovered. Some examples are illustrated in Figure~\ref{fig:determ_prob}(a) and Figure~\ref{fig:determ_prob}(c).

\begin{figure}
    \centering
    \includegraphics[page = 4]{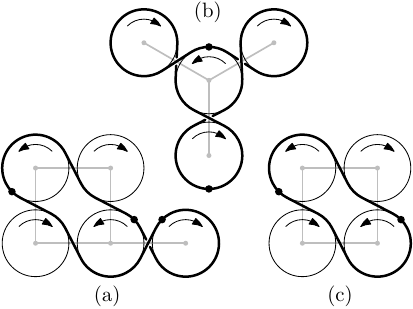}
    \caption{Examples of SCSs where the robots apply the deterministic strategy. The path traversed by the available robots is drawn in bold stroke. The robots are represented by solid black points. (a) The robots are not isolated, however there are uncovered trajectory segments. (b) The robots are isolated and everything is covered. (c) The robots are isolated and there are uncovered trajectory segments.}
    \label{fig:determ_prob}
\end{figure}

To cope these bad situations, a modification of the deterministic strategy of D\'iaz-B\'a\~nez et al.~$\cite{diaz2017}$ has been proposed by Caraballo~\cite{caraballo2019patrolling}. The idea is to perform the shifting operation in a random way.
Some simulations to compare random and deterministic strategies were performed in that preliminary work. 
The contribution of this paper with respect to~\cite{caraballo2019patrolling} is twofold: first, we perform a theoretical study of the stochastic model and then we present extensive computational experiments to validate both the teoretical results and the approach.


It is worth noting that our study is focused on the simple scenario introduced by  D\'iaz-B\'a\~nez et al.~$\cite{diaz2017}$, where the trajectories are considered as unit circles $C_1,\dots,C_n$, and the position of a robot in its trajectory is denoted by the central angle with respect to the positive horizontal axis. Thus, in a synchronization schedule $(\alpha,\delta)$, the starting position $\alpha_i\in \alpha$ and the travel direction $\delta_i\in\delta$ in trajectory $C_i$ are in $[0,2\pi)$ and $\{1,-1\}$, respectively (1 indicates counterclockwise direction and $-1$ clockwise direction). Therefore, the position of a robot in trajectory $C_i$ at time $t$ is given by $\alpha_i+t\cdot 2\pi\cdot \delta_i$. Nevertheless, our results can be generalized for more general trajectories as we will mention later.

The remainder of the paper is organized as follows. Section~\ref{sec:pre_work} briefly describes previous
work on multiagent patrolling using random walks; Section~\ref{sec:strategies-measures} introduces both the strategies and measures to be used; Section~\ref{theoretical} is devoted to the theoretical study of the random strategy based on the random walks theory.
Section~\ref{experiments} presents extensive simulations to evaluate the random strategies; and, in Section~\ref{sec:conclusions}, the
results and directions for future work are discussed.



\section{Related work}\label{sec:pre_work}
Recently, the study of stochastic multi-UAV systems has attracted considerable
attention in the field
of mobile robots. This approach has
several advantages such as shorter times to complete tasks, cost reduction, higher scalability, and more
reliability, among others \cite{yanmaz2010discrete,evers2014online}.
In a random mobility model, each robot randomly selects its direction, speed, and time independently of other robots. Some models include
random walks, random waypoints or random directions. See  the survey of Camp et al.~\cite{camp2002survey} for a comprehensive account. 

This paper assumes the framework of  D\'iaz-B\'a\~nez et al.~$\cite{diaz2017}$, then our model is not a pure random walk model. However,  the proposed random strategies generate random walks. A random walk on a graph is the process of visiting the nodes of the graph in some sequential random order. 
The walk starts at some fixed node, and at each step it moves to a neighbor of the current node chosen randomly. 
There is a vast theoretical literature dealing with random walks. For an
overview see e.g. the papers by Lov\'asz~\cite{lovasz} and R\'ev\'esz~\cite{revesz2013random}. 
One of the main reasons that random walk techniques are so appealing for networking applications is their robustness to dynamics.
It has been noted that random walk presents locality, simplicity, low-overhead and robustness to structural changes~\cite{chenthesis}. Because of these characteristics, applications based on random walks are becoming more and more popular in the networking community. In recent years, different authors have proposed the use of random walks for querying/searching, routing and self-stabilization in wireless networks, peer-to-peer networks, and other distributed systems~\cite{avin2008power}.

Relevant works in robotics are using criteria involving instants of visits of each node \cite{machado2002multi,srivastava2013stochastic}, where the main objective is to reduce the period between two consecutive visits to any vertex in order to minimize the detection delay for intruders. In the patrolling problem for communication scenarios, some models have been considered by Pasqualetti et al. ~\cite{pasqualetti2012cooperative} but, the evaluation criterion is also related to instants of visits. In this paper, two new measures are added to a criteria based on instants of visits to precisely ensure communication of the team dealing with random strategies:
the isolation time and the broadcast time. Recently, the concept of Flying Ad-Hoc Network (FANET), which is basically an ad hoc network between UAVs, has been introduced by Bekmezci  et al.~\cite{bekmezci2013flying}. FANETs can be seen as a subset of the well-known mobile ad hoc networks (MANETs), where the use of random walks has been intensive \cite{gupta2013performance}. Thus, an opportunity to extend the stochastic methods to UAVs has been opened. Our contribution is to
provide an example of using random strategies as an alternative to existing protocols for UAVs networks.

\section{Random strategies and measures}\label{sec:strategies-measures}

In the deterministic protocol proposed by D\'iaz-B\'a\~nez et al.~\cite{diaz2017}, a robot performs a shifting operation at the communication link every time it detects the absence of a neighbor. Two alternative random strategies are proposed in this work.

\newcommand{\rndstr}{\textbf{random\;}}
\emph{The random strategy:} Every time a robot arrives at a communication link between its current trajectory and a neighboring one, it chooses independently to remain in its trajectory or to pass to the neighboring trajectory by means of a shifting operation. The probability of remaining or shifting is $p=\frac{1}{2}$. Note that by following this protocol, more than one robot may share the same position. For instance, if two robots arrive at a communication link and one of them decides to maintain its trajectory and the other one decides to make a shifting operation, then they will move together like a single robot.
Note that in the basic model at hand we are considering an abstraction in which the trajectories are tangent and the communication link is a point. However, from the practical point of view, we can assume a kind of interval so that the robots have a safety margin to avoid collision.


\newcommand{\qrndstr}{\textbf{quasi-random\;}}
\emph{The quasi-random strategy:} Every time a robot arrives at a communication link, the following rule is applied: if there is no robot in the neighboring trajectory, then it decides whether to remain in its trajectory or to pass to the neighboring one by a shifting operation (the probability of remaining or shifting is $p=\frac{1}{2}$). Otherwise; that is, if there is a robot in the neighboring trajectory, it remains in its trajectory. Note that with this protocol, no two robots will travel on the same trajectory and a collision avoidance strategy is not required in this case.

\newcommand{\detstr}{\textbf{deterministic}}
\begin{definition}
	A synchronized system where the robots are applying the 
	random strategy is called a \emph{randomized SCS (R-SCS)}. A \emph{quasi-randomized SCS (QR-SCS)} is defined
	analogously. 
\end{definition}

\begin{remark}\label{rmk:movement-dependency}
	 The behavior of a robot in an R-SCS is independent of the other robots; 
	the movement of an robot in a QR-SCS depends on the behavior of the other robots.
	The behavior of robots following the deterministic strategy is totally co-dependent.
\end{remark}


In what follows, three criteria are introduced to compare the performance of these strategies. The metrics carry valuable information about the coverage and communication performance of an SCS.

\begin{itemize}
	\item  The \textbf{idle time} is the average time that a point in the union of the trajectories remains unobserved by a robot.

	\item The  \textbf{isolation time} is the average time that a robot is without communication with any other robot.
	
	\item The  \textbf{broadcast time} is the average time elapsed from the moment a robot emits a message until it is received by all the other robots.
\end{itemize}

It is worth noting that these metrics are somehow related to some resilience measures that have recently been defined for SCS's:
the idle time is a coverage measure and then it is related to the coverage-resilience,  defined by Caraballo~\cite{ThesisEvaristo} and Bereg et al.~\cite{bereg2020robustness} as the minimum number of
robots whose removal may result in a non-covered subarea, that is, the idle time is infinity.  Similarly, the isolation time is related to the 1-isolation resilience, defined by Caraballo~\cite{ThesisEvaristo} and Bereg et al.~\cite{bereg2018computing}  as the
cardinality of a smallest set of robots whose failure is sufficient to cause that at least one surviving robot to operate without communication. Finally,
the broadcast time is related to the broadcasting resilience, defined by Caraballo~\cite{ThesisEvaristo} and Bereg et al.~\cite{bereg2020robustness} as the minimum number of robots whose removal may
disconnect the network.

\section{Theoretical results}\label{theoretical}

In this section, some theoretical bounds on the above metrics are presented and then, these values are compared with experimental results. First, we review some well-known notions of random walks since random walks are the main tool used in this paper
to study the behavior of the aforementioned metrics in a SCS. We follow the notation of Lov\'asz~\cite{lovasz}. Let $G=(V, E)$ be a connected (di)graph with $n$ vertices. The process to generate a random walk is as follows. Starting at
a vertex $v_0$, randomly select a neighbor $v_1$ of $v_0$. Then randomly select a neighbor of $v_1$ and repeat the rule. This process yields a sequence of vertices of $G$, which it is called a \emph{random walk} on $G$. Let $v_t$ denote the vertex of the random walk at the $t$-th step. Let $P_0$ be the probability distribution from which $v_0$ was chosen; i.e., $P_0(v)=\mathrm{Pr}(v_0=v)$ for all $v\in V$. Denote by $M=(p_{vw})_{v,w\in V}$ the transition matrix, where $p_{vw}$ is the probability of moving from vertex $v$ to vertex $w$. Then, $P_t=M^tP_0$ denotes the probability distribution of $v_t$; that is, $P_t(v)=\mathrm{Pr}(v_t=v)$ for all $v\in V$. 

The \emph{period} $k$ of a vertex $v$ is defined as $k=\text{gcd}\{t > 0 \vert \Pr(v_t = v \vert v_0 = v) > 0 \}$. In other words, $v$ has period $k$ if it is maximal such that any subsequent visit to $v$ can occur in multiples of $k$ time steps. If $k=1$, the vertex $v$ is \emph{aperiodic}. If every vertex in $G$ is aperiodic, $G$ is said to be \emph{aperiodic}. 

A probability distribution $P$ that satisfies $P = M\cdot P$ is called a \emph{stationary distribution}. If the graph $G$ is (strongly) connected, then the stationary distribution exists and it is unique. Moreover, if $G$ is also aperiodic, $P_t$ converges to the stationary distribution $P$ as $t \rightarrow \infty$.




\subsection[Discretizing a partial SCS]{The discrete model in a partial SCS}\label{sec:discretization}

Let $(\alpha, \delta)$ be the synchronization schedule of a partial SCS $\mathcal{F}$ with $n$ trajectories $\{C_1,\dots,C_n\}$. Informally, 
$(\alpha, \delta)$ can be seen as a snapshot of a moving synchronized system.
In order to apply the random walks theory we must discretize our model. 
Consider the
digraph $G = (V, E)$, where $V=\{1,\dots,n\}$ and $(i, j)$ is an arc of $E$ from $i$ to $j$ 
if there is a path $p$ 
of length $2\pi$ from $\alpha_i$ to $\alpha_j$  following the assigned travel directions in the circles. That is, consider a vertex per circle and connect two vertices $i,j$ if a robot can travel from $i$ to $j$ in one unit of time\footnote{A path of length $2\pi$ is traveled in one unit of time in the model.}. See Figure~\ref{fig:discretization-SCS}, where an SCS is shown and a path between the starting positions $\alpha_6$ and $\alpha_5$ is marked in red.
Without loss of generality, it can be assumed that the starting positions are not link positions. Note that if some
starting positions are on a
communication link, then taking the
trajectory points after (or before)
$\epsilon$ units of time as starting
positions, an equivalent
synchronization schedule is obtained. 
Table~\ref{tab:adjacency-transitions}a shows the adjacency matrix of $G$ corresponding to the SCS of Figure~\ref{fig:discretization-SCS}. Note that the diagonal entries $(i,i)$ have value equal to 1 because a robot can make a tour along its trajectory and come back to the same starting position in one unit of time. We call $G$ the \emph{discrete motion graph} of the system.

Due to synchronization, if a robot is at some position $\alpha_i$ in  
$\alpha$ then all other robots are also at some position $\alpha_j$ in  
$\alpha$ (regardless the strategy being used).
Moreover, if a robot is at a position in $\alpha$ then after one unit of time, it will also be at some position in $\alpha$.
The movement of a robot in $\mathcal{F}$ can be modeled by a `walk' on $G$ and one step taken by the `walker' on $G$ corresponds to one time unit on $\mathcal{F}$. 
Also, 
given a sequence of edges (a path) on $G$ traversed by a walker, the path traversed by the corresponding robot 
in $\mathcal{F}$ is known. Thus, the behavior of $k$ available robots in $\mathcal{F}$ is modeled by $k$ simultaneous walkers on $G$.


%
In the next section it is proved that there exists at most one path of length $2\pi$ between two starting points in the synchronization schedule $(\alpha, \delta)$. Thus, the 
random strategy can be modeled using standard random walks on $G$. This does not happen for quasi-random- or deterministic- strategies, since the movement of an agent in $G$ depends on the motion of the other agents of the system (Remark~\ref{rmk:movement-dependency}). 

\subsection{Uniqueness of paths}

The following results are constrained to the simple model considered in this paper, that is, the trajectories are pairwise non-intersecting unit circles  and there exists a communication link between two circles if they are tangent to each other (they have a single common point).

In some cases, to simplify the notation, we may use the term SCS to actually refer the communication graph of a SCS. For example, when we are referring to a cycle in the communication graph of an SCS we will simply say: \emph{a cycle in an SCS}.
    \begin{figure}
        \centering
        \begin{minipage}{.3\textwidth}
            \centering
            \includegraphics[scale=.3, page=1]{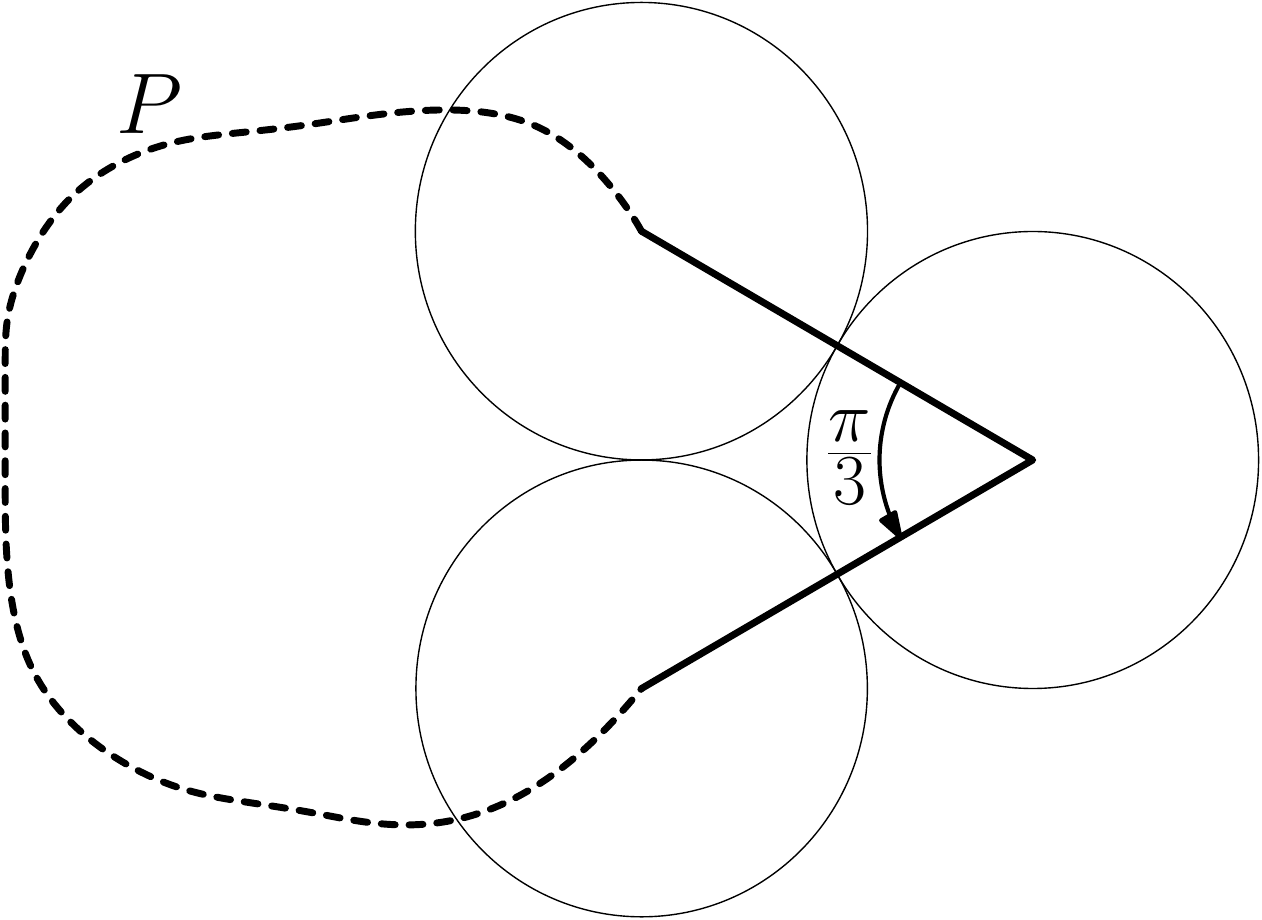}\\
            (a)
        \end{minipage}\qquad
        \begin{minipage}{.3\textwidth}
            \centering
            \includegraphics[scale=.3, page=2]{inner_polygon_angles.pdf}\\
            (b)
        \end{minipage}
        \begin{minipage}{.35\textwidth}
            \centering
            \includegraphics[scale=.7]{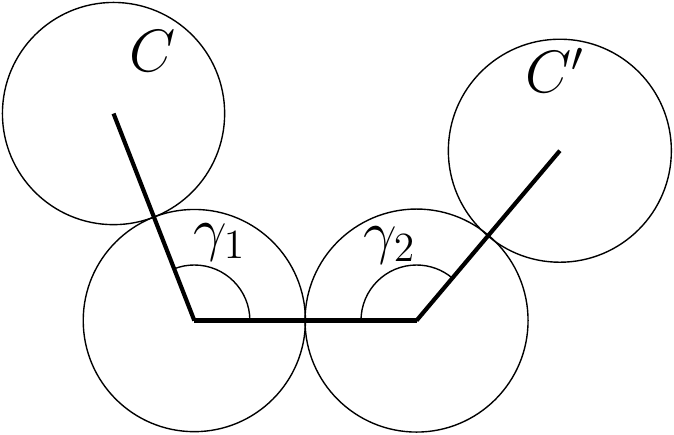}\\
            (c)
        \end{minipage}
        \caption{(a) The minimum amplitude of an inner angle in $P$ is $\frac{\pi}{3}$. (b) The maximum amplitude of an inner angle in $P$ is $\frac{5\pi}{3}$. Note that outer angles of $P$ fulfill the same restrictions. (c) Two consecutive angles at the same side of a path of length 3.}
        \label{fig:restrictions_in_P}
    \end{figure}


    \begin{remark}\label{rmk:in-polygon}
        Let $\mathcal{C}$ be a cycle in an SCS. 
        
        \begin{enumerate}
            \item The polygon $P$ formed by connecting the centers of consecutive circles in $\mathcal{C}$ is simple and their sides have length 2. Note that the sides of the polygon are edges of the communication graph.
            \item Every inner or outer angle $\gamma$ of $P$ fulfills: $\frac{\pi}{3}\leq \gamma \leq \frac{5\pi}{3}$. 
            See Figures \ref{fig:restrictions_in_P}(a) and \ref{fig:restrictions_in_P}(b).
            \item Two consecutive (inner or outer) angles $\gamma_1$ and $\gamma_2$ at the same side of a path of length 3 on the boundary of $P$ fulfill that $\gamma_1+\gamma_2\geq \pi$. Note that if $\gamma_1+\gamma_2<\pi$ then $C$ and $C'$ are not disjoint. See Figure~\ref{fig:restrictions_in_P}(c) for an illustration. 
        \end{enumerate}
        
    \end{remark}

    \begin{lemma}\label{lem:cycle}
       Let $\mathcal{C}$ be a cycle in an SCS. Let $\ell_1$ and $\ell_2$ be two different communication links in $\mathcal{C}$ such that two paths from $\ell_1$ to $\ell_2$ in $\mathcal{C}$ have the same length $L$. Then $L\geq 2\pi$.
    \end{lemma}
    
     \begin{figure}[!h]
        \centering
        \begin{minipage}{.33\textwidth}
            \centering
            \includegraphics[width=\columnwidth, page=2]{thm-proof.pdf}\\
            (a)
        \end{minipage}\qquad
        \begin{minipage}{.33\textwidth}
            \centering
            \includegraphics[width=\columnwidth, page=1]{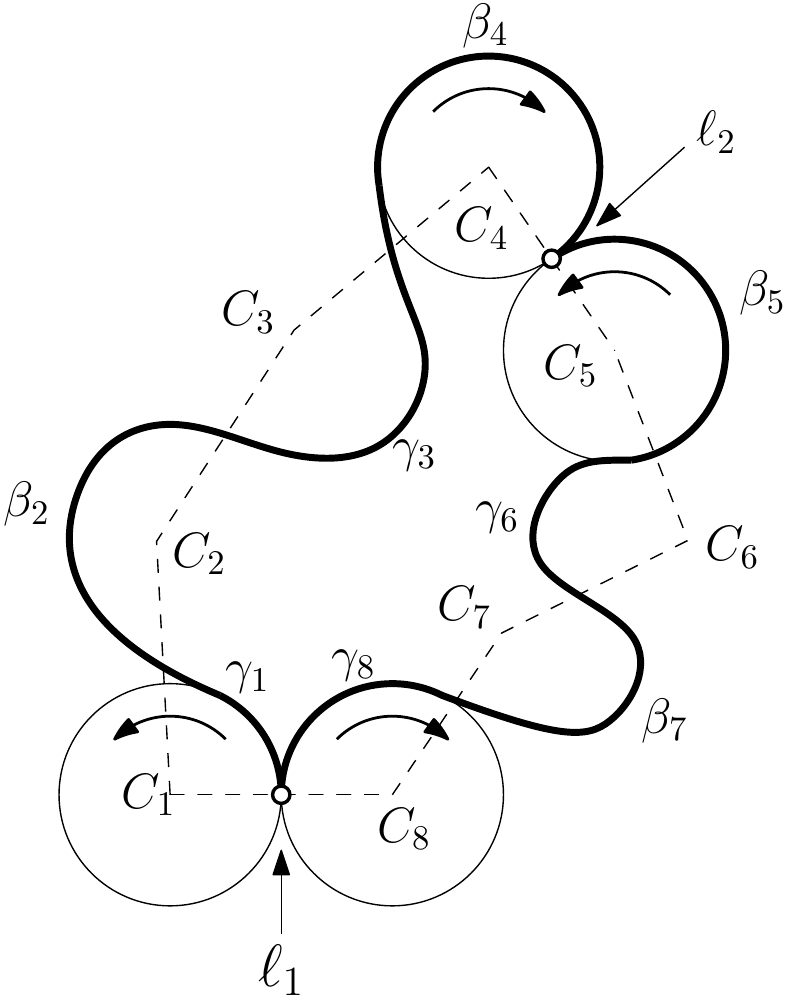}\\
            (b)
        \end{minipage}\\
        \begin{minipage}{.33\textwidth}
            \centering
            \includegraphics[width=\columnwidth, page=4]{thm-proof.pdf}\\
            (c)
        \end{minipage}\qquad
        \begin{minipage}{.33\textwidth}
            \centering
            \includegraphics[width=\columnwidth, page=3]{thm-proof.pdf}\\
            (d)
        \end{minipage}
        \caption{Two paths between two communication links $\ell_1$ and $\ell_2$ in a cycle of an SCS. There are only four different possibilities and they are determined by the travel direction of the circles adjacent to $\ell_1$ and $\ell_2$.}
        \label{fig:2-paths}
    \end{figure}

    \begin{proof}
        Let $P$ be the simple polygon formed by connecting the centers of consecutive circles in $\mathcal{C}$. Let $C_1, C_2,\dots, C_{2k}$ be the sequence of circles in $\mathcal{C}$ enumerated from $\ell_1$ in counterclockwise direction, see Figure~\ref{fig:2-paths}. Note that $\mathcal{C}$ has even length. Let $p_1$ be the path from $\ell_1$ to $\ell_2$ that starts traveling in $C_1$. Let $p_2$ be the path from $\ell_1$ to $\ell_2$ that starts traveling in $C_{2k}$. For the sake of contradiction, assume that $|p_1|=|p_2|=L<2\pi$. Let $A_1$  and $B_1$ be the sum of the inner and outer angles turned by $p_1$ in $P$, respectively. Let $A_2$  and $B_2$ be the sum of the inner and outer angles of $P$ in $p_2$, respectively. Then we have the following:
        
        \begin{align}
            |p_1|=A_1+B_1=L<2\pi, \label{eq:p_1=L}\\
            |p_2|=A_2+B_2=L<2\pi, \label{eq:p_2=L}\\
            A_1+A_2+B_1+B_2<4\pi.
            \label{eq:sum-p_1-p_2}
        \end{align}


        Depending on the travel direction of the circles adjacent to $\ell_1$ and $\ell_2$ four cases can be distinguished:
        \begin{enumerate}
            \item $p_1$ and $p_2$ depart from $\ell_1$ traversing inner angles of $P$, and arrive at $\ell_2$ traversing inner angles of $P$ (Figure~\ref{fig:2-paths}a).
            \item $p_1$ and $p_2$ depart from $\ell_1$ traversing inner angles of $P$, and arrive at $\ell_2$ traversing outer angles of $P$ (Figure~\ref{fig:2-paths}b).
            \item $p_1$ and $p_2$ depart from $\ell_1$ traversing outer angles of $P$, and arrive at $\ell_2$ traversing inner angles of $P$ (Figure~\ref{fig:2-paths}c).
            \item $p_1$ and $p_2$ depart from $\ell_1$ traversing outer angles of $P$, and arrive at $\ell_2$ traversing outer angles of $P$ (Figure~\ref{fig:2-paths}d).
        \end{enumerate}

        Cases 2 and 3 are analogous. If we reverse the travel direction on every cycle, case 2 becomes case 3 and viceversa. Then, we will focus on cases 1, 2 and 4.

        Let us denote by $\gamma_i$ (resp. $\beta_i$) the angle of $C_i$ which is inside (resp. outside) $P$. We use the same notation for the corresponding arcs. Note that:

        \begin{equation}
            \gamma_i+\beta_i=2\pi. \label{eq: complement-angles}
        \end{equation}

        \textbf{Case 1:} (Figure~\ref{fig:2-paths}(a))
        \begin{align*}
            p_1 &=\gamma_1,\beta_2,\gamma_3,\dots,\gamma_{2r+1}\\
            p_2 &= \gamma_{2r+2}, \beta_{2r+3}, \dots, \gamma_{2k}.
        \end{align*}

        Thus,
        \begin{align*}
            A_1&=\gamma_1+\gamma_3+\dots+\gamma_{2r+1},\\ 
            B_1&= \beta_2+\beta_4+\dots+\beta_{2r},\\
            A_2&=\gamma_{2r+2}+\gamma_{2r+4}+\dots+\gamma_{2k}, \text{ and} \\
            B_2&= \beta_{2r+3}+\beta_{2r+5}+\dots+\beta_{2k-1}.
        \end{align*}

        By other hand, the sum of the inner angles of $P$ is $\pi(2k-2)=2\pi(k-1)$.\footnote{The sum of the inner angles of a simple polygon of $n$ sides is $\pi(n-2)$ In our case, the polygon $P$ has $n=2k$ sides.} Note that $B_1$ and $B_2$ are the sum of $r$ and $(k-r-1)$ outer angles of $P$, respectively. Using equation (\ref{eq: complement-angles}) we have that:
        \begin{align}
            A_1+2r\pi-B_1+A_2+2(k-r-1)\pi-B_2&=2\pi(k-1)\nonumber\\
            A_1+A_2-B_1-B_2&=0\nonumber\\
            A_1+A_2&=B_1+B_2 \label{eq:case1-sum_P}.
        \end{align}

        The angles $\gamma_1$ and $\gamma_{2k}$ are consecutive in $P$, then $\gamma_1+\gamma_{2k}\geq \pi$ (Remark~\ref{rmk:in-polygon}, 3rd item). Also, the angles $\gamma_{2r+1}$ and $\gamma_{2r+2}$ are consecutive in $P$, then $\gamma_{2r+1}+\gamma_{2r+2}\geq \pi$. Therefore: $$A_1+A_2\geq \gamma_1+\gamma_{2r+1}+\gamma_{2r+2}+\gamma_{2k}\geq 2\pi.$$ 
        Finally, from equation (\ref{eq:case1-sum_P}) we get that $A_1+A_2+B_1+B_2\geq 4\pi$ which is a contradiction with equation (\ref{eq:sum-p_1-p_2}) and the result is proved for Case 1.

        \vspace{.5cm}
        \textbf{Case 2:} (Figure~\ref{fig:2-paths}(b))

        \begin{align*}
            p_1 &=\gamma_1,\beta_2,\gamma_3,\dots,\beta_{2r}\\
            p_2 &= \beta_{2r+1}, \gamma_{2r+2}, \dots, \gamma_{2k}.
        \end{align*}

        Thus,

        \begin{align*}
            A_1&=\gamma_1+\gamma_3+\dots+\gamma_{2r-1},\\
            B_1&= \beta_2+\beta_4+\dots+\beta_{2r}, \\
            A_2&=\gamma_{2r+2}+\gamma_{2r+4}+\dots+\gamma_{2k}, \text{ and}\\
            B_2&= \beta_{2r+1}+\beta_{2r+5}+\dots+\beta_{2k-1}.
        \end{align*}

        Notice that $B_1$ and $B_2$ are the sum of $r$ and $(k-r)$ outer angles of $P$, respectively. From equation (\ref{eq: complement-angles}) we have that:
        \begin{align}
            A_1+2r\pi-B_1+A_2+2(k-r)\pi-B_2&=2\pi(k-1)\nonumber\\
            A_1+A_2-B_1-B_2&=-2\pi\nonumber\\
            A_1+A_2&=B_1+B_2-2\pi \label{eq:case2-sum_P}.
        \end{align}

        The angles $\gamma_1$ and $\gamma_{2k}$ are consecutive in $P$, then $\gamma_1+\gamma_{2k}\geq \pi$ (Remark~\ref{rmk:in-polygon}, 3rd item). Therefore:
        
        \begin{equation}\label{eq:sum_A}
            A_1+A_2\geq\pi,
        \end{equation}
        and using equation (\ref{eq:case2-sum_P}) we get that:
        \begin{equation}\label{eq:sum_B}
            B_1+B_2\geq 3\pi.
        \end{equation}
        Finally, adding (\ref{eq:sum_A}) and (\ref{eq:sum_B}) we deduce that $A_1+A_2+B_1+B_2\geq 4\pi$ which is a contradiction with equation (\ref{eq:sum-p_1-p_2}). Case 2 is done.


        \vspace{.5cm}
        \textbf{Case 4:} (Figure~\ref{fig:2-paths}(d))

        \begin{align*}
            p_1 &=\beta_1,\gamma_2,\beta_3,\dots,\beta_{2r+1}\\
            p_2 &= \beta_{2r+2}, \gamma_{2r+3}, \dots, \beta_{2k}.
        \end{align*}

        Thus,

        \begin{align*}
            A_1&=\gamma_2+\gamma_4+\dots+\gamma_{2r},\\
            B_1&= \beta_1+\beta_3+\dots+\beta_{2r+1}, \\
            A_2&=\gamma_{2r+3}+\gamma_{2r+5}+\dots+\gamma_{2k-1}, \text{ and}\\
            B_2&= \beta_{2r+2}+\beta_{2r+4}+\dots+\beta_{2k}.
        \end{align*}

        Notice that $B_1$ and $B_2$ are the sum of $(r+1)$ and $(k-r)$ outer angles of $P$, respectively. Then, using equation (\ref{eq: complement-angles}) we have that:
        \begin{align}
            A_1+2(r+1)\pi-B_1+A_2+2(k-r)\pi-B_2&=2\pi(k-1)\nonumber\\
            A_1+A_2-B_1-B_2+2\pi&=-2\pi\nonumber\\
            A_1+A_2&=B_1+B_2-4\pi \label{eq:case4-sum_P}.
        \end{align}

        Obviously $A_1+A_2\geq 0$, and from equation (\ref{eq:case4-sum_P}) we get: $B_1+B_2\geq 4\pi$. Thus, $A_1+A_2+B_1+B_2\geq 4\pi$, contradicting equation (\ref{eq:sum-p_1-p_2}) and the Lemma follows.
       
    \end{proof}

The following result has been proved by Bereg et al.~\cite{bereg2018computing}:

\begin{lemma}\label{lem:length}
In a SCS, the length of a path between any two starting positions in a synchronized schedule is in $2\pi\mathbb{N}$.
\end{lemma}

    \begin{theorem}\label{unicity}
        In an SCS, there exists at most one path of length $2\pi$ between two starting positions.
    \end{theorem}
    \begin{proof}
     As noted before, we can assume that the starting positions of an SCS are not in communication links. 
     For the sake of contradiction, let us suppose that there exist a pair of starting positions $\alpha_i$ and $\alpha_{i}'$ connected by two different paths $p$ and $p'$ of length $2\pi$. Let $\ell_1$ be the first communication link where $p$ and $p'$ separate from each other. Let $\ell_2$ be the first communication link after $\ell_1$ where $p$ and $p'$ join again, see Figure~\ref{fig:two_paths}. Let $d_p(x,y)$ (resp. $d_{p'}(x,y)$) denote the length of the section of $p$ (resp. $p'$) from $x$ to $y$. Also denote by $p(x,y)$ the section of the path $p$ from $x$ to $y$. By hypothesis, we have that $d_p(\alpha_i, \alpha_{i}')=d_{p'}(\alpha_i, \alpha_{i}')=2\pi$. By definition of $\ell_1$, we also have that $d_p(\alpha_i,\ell_1)=d_{p'}(\alpha_i, \ell_1)$. 
     Now, let us suppose that $d_p(\ell_1,\ell_2)\neq d_{p'}(\ell_1,\ell_2)$. Assume that $d_p(\ell_1,\ell_2)<d_{p'}(\ell_1,\ell_2)$. Then, $$d_p(\ell_2,\alpha_{i}')=2\pi-d_p(\alpha_i, \ell_1)-d_p(\ell_1,\ell_2)>2\pi-d_{p'}(\alpha_i, \ell_1)-d_{p'}(\ell_1,\ell_2)=d_{p'}(\ell_2,\alpha_{i}').$$
     As a consequence, the path obtained by joining the sections $p(\alpha_i,\ell_1),p'(\ell_1,\ell_2),p(\ell_2,\alpha_{i}')$ has length less than $2\pi$, contradicting Lemma \ref{lem:length}.
     A similar argument can be used if $d_p(\ell_1,\ell_2)>d_{p'}(\ell_1,\ell_2)$. Therefore, $d_p(\ell_1,\ell_2)=d_{p'}(\ell_1,\ell_2)<2\pi$. However, this is a contradiction with Lemma \ref{lem:cycle} and the result follows.
    \end{proof}

    \begin{figure}
        \centering
        \includegraphics{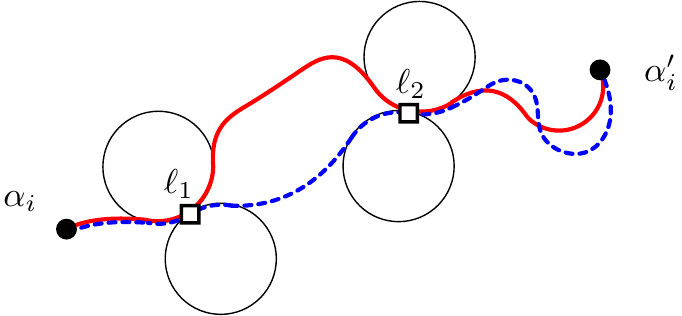}
        \caption{Two paths between two starting positions $\alpha_i$ and $\alpha{_i}'$.}
        \label{fig:two_paths}
    \end{figure}

\begin{remark}
   Note that for general configurations, considering non circular trajectories, Theorem \ref{unicity} could be not true. However, the discretization of the model can be generalized by considering a multigraph and uniform probability for all edges representing the $2\pi$-length trajectories between two points. In any case, an SCS that uses the 
   random strategy can be modeled by random walks for practical configurations for which Theorem \ref{unicity} holds.
\end{remark}

\subsection{Randomized SCS's}\label{sec:transition-theoretical}

\begin{figure}[h]
	\centering
		\centering
		\includegraphics[scale=1.5]{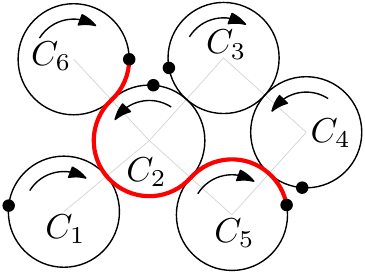}
	\caption{The solid points are the starting positions of a synchronization schedule. The path (in red) from $\alpha_6$ to $\alpha_5$ has length $2\pi$ and follows the travel directions along the visited trajectories. This path traverses four communication links. 
	}
\label{fig:discretization-SCS}
\end{figure}

\begin{table}[h]
	\centering
	\begin{subfigure}{.45\textwidth}
		\centering
		\def\arraystretch{1.3}%
		\begin{tabular}{l llllll}
			\toprule
			& 1 & 2 & 3 & 4 & 5 & 6\\
			\cmidrule{2-7}
			1 & 1 & 1 & 1 & 1 & 1 & 0\\
			2 & 1 & 1 & 1 & 1 & 1 & 1\\
			3 & 0 & 1 & 1 & 1 & 1 & 0\\
			4 & 0 & 1 & 1 & 1 & 1 & 0\\
			5 & 0 & 1 & 1 & 1 & 1 & 0\\
			6 & 1 & 1 & 1 & 1 & 1 & 1\\
			\bottomrule		
		\end{tabular}
		\caption{}
	\end{subfigure}
	\begin{subfigure}{.45\textwidth}
		\centering
		\def\arraystretch{1.3}%
		\begin{tabular}{l llllll}
			\toprule
			& 1 & 2 & 3 & 4 & 5 & 6\\
			\cmidrule{2-7}
			1 & $\frac{1}{2}$ & $\frac{1}{8}$ & $\frac{1}{8}$ & $\frac{1}{8}$ & $\frac{1}{8}$ & 0\\
			2 & $\frac{1}{4}$ & $\frac{1}{16}$ & $\frac{1}{16}$ & $\frac{1}{16}$ & $\frac{1}{16}$ & $\frac{1}{2}$\\
			3 & 0 & $\frac{1}{4}$ & $\frac{1}{4}$ & $\frac{1}{4}$ & $\frac{1}{4}$ & 0\\
			4 & 0 & $\frac{1}{4}$ & $\frac{1}{4}$ & $\frac{1}{4}$ & $\frac{1}{4}$ & 0\\
			5 & 0 & $\frac{1}{4}$ & $\frac{1}{4}$ & $\frac{1}{4}$ & $\frac{1}{4}$ & 0\\
			6 & $\frac{1}{4}$ & $\frac{1}{16}$ & $\frac{1}{16}$ & $\frac{1}{16}$ & $\frac{1}{16}$ & $\frac{1}{2}$\\
			\bottomrule
		\end{tabular}
		\caption{}
	\end{subfigure}
	\caption{(a) Adjacency matrix of the discrete motion graph on Figure~\ref{fig:discretization-SCS}. The value $0$ means that a robot needs to travel more that $2\pi$ to connect the corresponding starting positions. (b) Matrix of transition probabilities of the discrete motion graph on Figure~\ref{fig:discretization-SCS} following the random strategy.} 
	\label{tab:adjacency-transitions}
\end{table}

This section focuses on a theoretical study of the 
random strategy using random walks on the discrete motion graph $G$.
Let $G=(V,E)$ be the discrete motion graph of an R-SCS with synchronization schedule $(\alpha,\delta)$. Let $\alpha_i$ and $\alpha_j$ be the starting positions of trajectories $C_i$ and $C_j$, respectively. If $(i,j)\in E$, there exists a path $p$ of length $2\pi$ from $\alpha_i$ to $\alpha_j$. Let $c_{ij}$ be the number of communication links traversed by $p$. As an illustration, in Figure~\ref{fig:discretization-SCS} the path from $\alpha_6$ to $\alpha_5$ traverses four communication links and $c_{6,5}=4$. 
The transition matrix of the process $M=(p_{ij})_{i,j\in V}$ is defined as follows:

\[M_{ij}=\displaystyle\left\{\begin{array}{cl}
(1/2)^{c_{ij}} &  \text{if }(i,j)\in E\\
0 & \text{otherwise}.
\end{array}\right.\]


Table~\ref{tab:adjacency-transitions}b shows the transition matrix associated with the discrete motion graph of the SCS shown in Figure~\ref{fig:discretization-SCS}. Using the definition of $M$ and taking into account that there are two possible paths to follow at every communication link, the following is immediately derived.
 
 \begin{property}\label{prop:stochastic-matrix}
 	The matrix $M$ is a right stochastic matrix; i.e., $\sum_j M_{ij}=1$. 	
 \end{property}





\begin{lemma}\label{lem:doubly-stoch}
	The transition matrix $M$ associated to a random walk on a discrete motion graph is \emph{doubly stochastic}. That is,
	\[\sum_{i} M_{ij}=\sum_{j} M_{ij}=1.\]
\end{lemma}

\begin{proof}
	Let $S=(\alpha,\delta)$ be the synchronization schedule used in the system.	By Property~\ref{prop:stochastic-matrix} we have $\sum_j M_{ij}=1$. Now, let us prove that $\sum_i M_{ij}=1$.
	
	Let $S'=(\alpha,-\delta)$ be the 
	synchronization schedule  that results from reversing the orientation of every circle and maintaining the synchronized starting points. Let $G'=(V,E')$ be the discrete motion graph of the system using $S'$ and let $M'$ be the transition matrix associated with a random walk on $G'$. Note that $G$ and $G'$ have the same set of vertices $V$ and $(i,j)\in E$ if and only if $(j,i)\in E'$. Also note that $M_{ij}=M'_{ji}$, and then $M'=M^\intercal$. From Property~\ref{prop:stochastic-matrix} we know that $\sum_j M'_{ij}=1$; that is, each row of $M'$ adds up to $1$. Therefore, each column of $M$ adds up to $1$ and the result follows.
\end{proof}


\begin{lemma}\label{lem:strongly-connected}
The discrete motion graph of a synchronized system is strongly connected.
\end{lemma}
\begin{proof}
	We can prove this lemma by induction on the number of trajectories of the system. It is easy to check that the lemma holds for a synchronized system with one trajectory or with two trajectories. Suppose, as inductive hypothesis, that the lemma holds for any synchronized system of $n$ trajectories, $n\geq 2$. 
	
	Let $\mathcal{F}$ be a synchronized system with $n+1$ trajectories. Let $G$ be the discrete motion graph of $\mathcal{F}$. Let $T$ be a spanning tree of the communication graph of $\mathcal{F}$. Let $C_i$ and $C_j$ be two trajectories of the system such that $C_i$ corresponds to a leaf of $T$ and $C_j$ is adjacent to $C_i$ in $T$, see Figure~\ref{fig:strongly-connected}. Note that by removing $C_i$ from the system and keeping the starting positions and travel directions in the other trajectories, we get a synchronized system $\mathcal{F'}$ with $n$ trajectories. Let $G'$ be the discrete motion graph of $\mathcal{F'}$. Let $i'$ be an arbitrary vertex of $G'$. By the inductive hypothesis there are two paths in $G'$, one from $i'$ to $j$ and the other from $j$ to $i'$. Then there are paths in $\mathcal{F'}$ from $\alpha_{i'}$ to $\alpha_j$ and viceversa, Figure~\ref{fig:strongly-connected} shows these paths in red in subfigures (a) and (b), respectively. Observe that, in $\mathcal{F}$ there are paths of length $2\pi$ from $\alpha_i$ to $\alpha_j$ and vice versa (these paths are shown in blue in Figure~\ref{fig:strongly-connected}, (a) and (b), respectively). Therefore, the arcs $(i,j)$ and $(j,i)$ are in $G$. As a consequence, the path from $i'$ to $j$ (resp. from \ $j$ to $i'$) in $G'$ can be extended using the edge $(j,i)$ (resp. \ $(i,j)$) and the lemma is fulfilled.
\end{proof}

\begin{figure}
	\centering
	\begin{subfigure}{.45\textwidth}
		\centering
		\includegraphics[page=1]{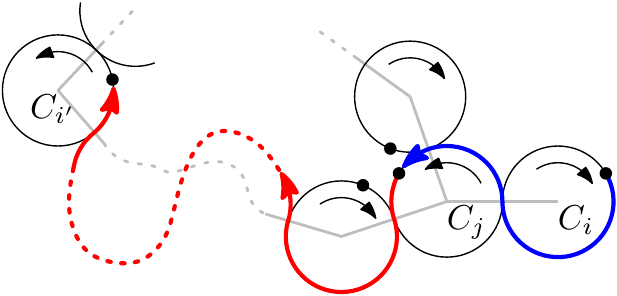}
		\caption{}
	\end{subfigure}\qquad
	\begin{subfigure}{.45\textwidth}
		\centering
		\includegraphics[page=2]{strongly-connected.pdf}
		\caption{}
	\end{subfigure}
	\caption{Illustration of the proof of Lemma~\ref{lem:strongly-connected}. (a) In red: path from $\alpha_j$ to $\alpha_{i'}$. In blue: path from $\alpha_i$ to $\alpha_j$. (b) In red: path from $\alpha_{i'}$ to $\alpha_{j}$. In blue: path from $\alpha_j$ to $\alpha_i$.}
	\label{fig:strongly-connected}
\end{figure}

From the previous lemma, the following is directly deduced.
\begin{corollary}
A random walk in a discrete motion graph is \emph{irreducible}
\footnote{A random walk in a graph is irreducible if, for every pair of vertices $u$ and $v$, the probability of transiting from $u$ to $v$ is greater than zero \cite{lovasz}.}.
\end{corollary}

Also, from Lemma
\ref{lem:strongly-connected} and taking into account that $M_{ii}>0$ for all $i$, we obtain the following.
\begin{corollary}\label{cor:aperiodic}
	A random walk in a discrete motion graph is \emph{aperiodic}.
\end{corollary}

Now, from Corollary~\ref{cor:aperiodic}, Lemma~\ref{lem:strongly-connected} and Lemma~\ref{lem:doubly-stoch} we arrive at the main result of this section.

\begin{theorem}\label{thm:stationary}
	A random walk on a discrete motion graph has a stationary distribution $\pi^*$, and $\pi^*$ is uniform. That is, $\pi^*(i)=1/n$ for all $1\leq i\leq n$, where $n$ is the number of trajectories in the system.
\end{theorem}


From Theorem~\ref{thm:stationary} and using the properties of a synchronized system, the following result is a direct consequence.

\begin{corollary}\label{cor:stationary_SCS}
	Let $\mathcal{F}$ be an R-SCS. Let $(\alpha,\delta)$ denote the synchronization schedule used on $\mathcal{F}$.  Assuming that the starting position of a robot is chosen with uniform probability in $\alpha=\{\alpha_1,\dots, \alpha_n\}$, then after $t\in \mathbb{R}^+$ units of time, the robot is at some point of $\alpha'=\{\alpha_1+2\pi \delta_1 t,\dots,\alpha_n+2\pi \delta_n t\}$ with uniform probability. Moreover, $(\alpha',\delta)$ is a new synchronization schedule (which is equivalent
	to $(\alpha,\delta)$).
	
\end{corollary}

Let $\mathcal{F}$ be an R-SCS and let $G=(V,E)$ be its discrete motion graph.
Knowing that the stationary distribution of a random walk on $G$ is uniform, in the following two subsections some theoretical results about the idle time and the isolation time of $\mathcal{F}$ are presented. Obtaining theoretical results for the broadcast time for our model seems to be hard and remains open for future works. Indeed, previous work on broadcast time is limited to specific
graph families, namely grids (using a protocol different from ours, since all possible directions at each vertex are considered in the literature) and random graphs. See the paper by Giakkoupis et al.~\cite{saribekyan2019spread} and references therein for a recent work on the broadcast time.

\subsubsection{Idle time}\label{sec:idle}

In this subsection, some results related to the idle time of an R-SCS with $k$ robots are presented.

\newcommand{\stepscount}{\lceil\frac{n}{k}\rceil}
\begin{theorem}\label{thm:patrolling}
	Suppose that $k$ agents are performing a random walk on a discrete motion graph $G=(V,E)$ and their starting vertices are taken uniformly on $V$. Let $v$ be an arbitrary vertex of $V$. The expected number of robots that visit $v$ during an interval of $\stepscount$ steps is at least one $(n=|V|)$.
\end{theorem}
\begin{proof}
	Let $t_1,\dots,t_c$ be an arbitrary sequence of $c=\stepscount$ steps of the random walk ($t_i=t_1+i-1$ for all $1<i\leq c$). By Theorem~\ref{thm:stationary}, at step $t_i$ ($1\leq i \leq c$) each robot is at $v$ with probability $1/n$. For every $1\leq i \leq c$ and every $1\leq j\leq k$, let $X_{ij}$ be the random variable given by
	\begin{equation*}
	X_{ij} = 
	\begin{cases}
	1 & \text{if the }j\text{-th robot is at } v \text{ at step } t_i,\\
	0 & \text{otherwise.}
	\end{cases} 
	\end{equation*}
	Therefore, \[ E[X_{ij}]=\frac{1}{n}.\]
	Let \[Y=\sum_{i=1}^{c}\sum_{j=1}^{k} X_{ij}.\]
	Note that the expected number of robots that visit $v$ during the given interval
	of steps is equal to $E[Y]$.
	By linearity of expectation, 
	\[E[Y]=\sum_{i=1}^{c}\sum_{j=1}^{k} E[X_{ij}]=\sum_{i=1}^{c}\sum_{j=1}^{k} \frac{1}{n}=c\frac{k}{n}=\left\lceil\frac{n}{k}\right\rceil\frac{k}{n}\geq 1.\qedhere\]
\end{proof}



\begin{theorem}\label{thm:idle time}
	Let $\mathcal{F}$ be an R-SCS where $k$ robots are operating. If the robots start at randomly chosen positions then the idle time is at most $\frac{n}{k}+1$, where $n=|V|$.
\end{theorem}
\begin{proof}
    Let $q$ be point in the union of the trajectories. Without loss of generality assume that $q \in C_1$.
    Let $t^\ast \in [0,1)$ be such that $q=\alpha_1+\delta_1 2 \pi t^\ast.$ For every integer
    $i \ge 0$, let $p$ be the probability that at time $t^\ast+i$ there is a robot at $q$. By Corollary~\ref{cor:stationary_SCS},
    for every $1 \le j \le k$, the $j$th-robot is at $q$ with probability $1/n$. Therefore, $p$ is independent of the choice of $i$
    and 
\[    p =1-\left (1-\frac{1}{n} \right )^k
      \ge 1-\frac{1}{1+\frac{k}{n}}
      = \frac{\frac{k}{n}}{1+\frac{k}{n}}, \]
     where the inequality follows from the inequality 
     \begin{equation} \label{eq:ber}
        (1-x)^n \le \frac{1}{1+nx}, \ 0 \le x \le 1 \textrm{ and } n \in \mathbb{N}.
     \end{equation}
     Equation~\ref{eq:ber} follows from the AM-GM inequality. The AM-GM inequality states
     that the arithmetic mean (AM) of a list of non-negative real numbers is greater or equal to
     its geometric mean (GM). Thus, \[(1-x)^n(1+nx)=GM(1-x,1-x,\ldots,1-x,1+nx)^{n+1}\leq
AM(1-x,1-x,\ldots,1-x,1+nx)^{n+1}=1.\footnote{ see \texttt{https://math.stackexchange.com/questions/2078342/a-simpler-proof-of-1-xn-frac11nx}} \]
    Let $X$ be the geometric random variable with parameter $p$. We have that
    the expected time between two consecutive visits to $q$ by a robot is at most
    \[E[X]=\frac{1}{p}\le \frac{n}{k}+1.\]
\end{proof}


This result assumes that the robots are uniformly located at the beginning and that the system lies in the stationary distribution. However, in practical applications the robots are all deployed at a given position; for instance in the trajectories nearest to the boundary of the global region. In this case, the time it takes for the system to reach a stationary distribution; that is, the \emph{mixing time} must be studied. Informally, the mixing time is the time needed for a random walk to reach its stationary distribution, or, more specifically, the number of steps
before the distribution of a random walk is \emph{close} to its stationary distribution. In the experiments, it will be shown that the mixing time in the model is very small for grids in real scenarios and then it can be said that, after a few units of time (for instance, 5 steps for a $5\times 5$ grid; i.e., for 25 trajectories), the idle time of the random strategy for $k$ drones is at most $\stepscount$.

\subsubsection{Isolation time}\label{sec:isolation}

A \emph{meeting} occurs when two robots arrive at a common communication link between their trajectories at the same time. The expected time that a robot is isolated, that is, the expected time between two consecutive meetings that a robot has is studied in this section. This measure is theoretically studied in the discrete motion graph using resources from random walks as in the previous section. Therefore, in this theoretical study, a meeting between two agents when they visit the same vertex of the discrete motion graph at the same time is studied.

Let $\mathcal{F}$ be an R-SCS and let $G$ be its corresponding discrete motion graph.
Suppose that, in $G$, an agent $u$ has two consecutive meetings at time steps $t$ and $t+s$. Agent $u$ has thus been isolated in $G$ by $s$ time units. Let $v_0,\dots,v_s$ be the consecutive sequence of vertices of $G$ visited by $u$ in this time interval. Let $u'$ and $u''$ be the agents that meet $u$ in $v_0$ and $v_s$, respectively. Taking into account that the vertices of $G$ are not communication links, then in $\mathcal{F}$, $u$ and $u'$ separate their paths at a communication link $p$ traversed by them at some time $t'$ such that $t<t'<t+1$, see Figure \ref{fig:isolation-SCSvsDiscrete}. Similarly, $u$ and $u''$ meet at a communication link $q$ of $\mathcal{F}$ at some time $t''$ such that $t+s-1<t''<t+s$, 
Therefore, in $\mathcal{F}$, $u$ has been isolated for fewer than $s$ time units. Also, notice that, if two robots in $\mathcal{F}$ meet at a communication link ($o$ in Figure \ref{fig:isolation-SCSvsDiscrete}) coming from two different starting positions ($v_0$ and $v_0'$ in Figure \ref{fig:isolation-SCSvsDiscrete}) and, after the meeting, they go toward different starting positions ($v_1$ and $v_1'$ in Figure \ref{fig:isolation-SCSvsDiscrete}), then this meeting is not reported in $G$. 

\begin{figure}
\centering
\includegraphics[scale=1.2]{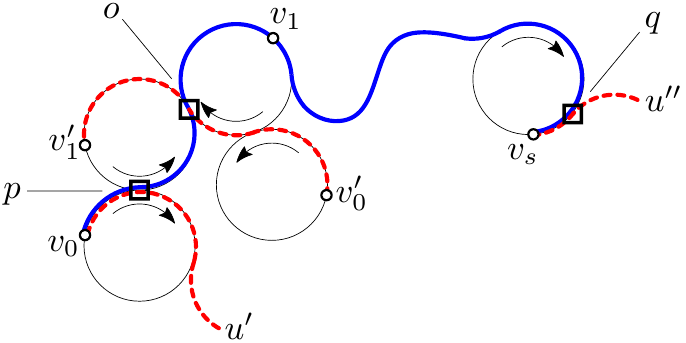}
\caption{The points $v_0,v_1,\dots,v_s, v_0', v_1'$ are the corresponding starting positions on their circles. The path $v_0,v_1,\dots, v_s$, traversed by a robot $u$, is shown in blue. The robot $u'$ departs with $u$ from $v_0$ at time $t$ and their paths separate at the communication link $p$. The robot $u''$ meets $u$ at the communication link $q$ and they arrive at $v_s$ traveling together. If a robot $u'''$ departs from $v_0'$ at time $t$ following the red dashed path that ends at $v_1'$, then, $u$ and $u'''$ meet at the communication link $o$ without visiting any starting position together.}
\label{fig:isolation-SCSvsDiscrete}
\end{figure}

From these observations, it is straightforward to deduce the following.
\begin{remark}\label{rmk:isolation-upper-bound}
	The isolation time of a robot in an R-SCS is always less than the isolation time in the corresponding discrete motion graph. 
\end{remark}




\begin{theorem}\label{thm:starvation}
	Suppose that $k$ agents are performing a random walk on a discrete motion graph $G=(V,E)$ and their starting vertices are taken uniformly on $V$.  Then the expected number of times that one agent meets another agent at some vertex of $V$ during an interval of $\displaystyle\left\lceil\frac{n^{k-1}}{n^{k-1}-(n-1)^{k-1}}\right\rceil$ steps is at least one.
\end{theorem}
\begin{proof}
	Let $t_1, \dots, t_c$ be an arbitrary sequence of $c=\displaystyle\left\lceil\frac{n^{k-1}}{n^{k-1}-(n-1)^{k-1}}\right\rceil$ steps of a random walk ($t_i=t_1+i-1$ for all $1<i\leq c$) performed by an agent $u$.
	By Theorem~\ref{thm:stationary}, it can be deduced that $u$ is alone at a vertex $v$ of $V$ at step $t_i$ with probability $\frac{1}{n}\left(\frac{n-1}{n}\right)^{k-1}=\frac{(n-1)^{k-1}}{n^k}$.
	Also, $u$ is not at $v$ at step $t_i$ with probability $\frac{n-1}{n}$. Thus, the probability of $u$ being with at least one robot at $v$ at step $t_i$ is given by: \[1-\frac{(n-1)^{k-1}}{n^k}-\frac{n-1}{n}=\frac{n^{k-1}-(n-1)^{k-1}}{n^k}.\]
	
	Let $X_{i,v}$ be the random variable given by:
	\begin{equation*}
	X_{i,v} = 
	\begin{cases}
	1 & \text{if } u \text{ meets some robot}  \text{ at } v \text{ at step } t_i,\\
	0 & \text{otherwise.}
	\end{cases} 
	\end{equation*}
	It is easy to see that \[ E[X_{i,v}]=\frac{n^{k-1}-(n-1)^{k-1}}{n^k}.\]
	Let \[Y=\sum_{i=1}^{c} \sum_{v \in V} X_{i, v}.\]

	Now, the expected number of times that $u$ meets another agent at some vertex of $V$ during the given sequence of steps is equal to $E[Y]$ and
	\[E[Y]=\sum_{i=1}^{c}\sum_{v \in V} E[X_{i,v}]=\sum_{i=1}^{c}\sum_{v \in V} \left(\frac{n^{k-1}-(n-1)^{k-1}}{n^k}\right)\geq 1.\qedhere\]
\end{proof}

From this theorem, the following result is obtained.
\begin{corollary}\label{cor:meeting_G}
	Suppose that $k$ agents are performing a random walk on a discrete motion graph $G=(V,E)$ and their starting vertices are taken uniformly on $V$.   For every agent, the expected number of steps between two consecutive meetings is at most $\displaystyle\left\lceil\frac{n^{k-1}}{n^{k-1}-(n-1)^{k-1}}\right\rceil$.
\end{corollary}

Now, extending the previous result to an R-SCS (using Remark~\ref{rmk:isolation-upper-bound}), we have the following.
\begin{corollary}\label{cor:meeting_F}
	Let $\mathcal{F}$ be an R-SCS where $k$ robots are operating. If the robots start at randomly chosen positions, then for every robot, the expected time between two consecutive meetings is less than $\displaystyle\left\lceil\frac{n^{k-1}}{n^{k-1}-(n-1)^{k-1}}\right\rceil$.
\end{corollary}

Note that Corollary~\ref{cor:meeting_F} says that a robot meets another robot in an interval of time of length at most $\displaystyle\left\lceil\frac{n^{k-1}}{n^{k-1}-(n-1)^{k-1}}\right\rceil$. This implies that when there are only two robots in the system, they meet at most every $n$ time units (by substituting $k=2$ in the formula). 

Figure~\ref{fig:theoretical_isolation} shows the behavior of the expected time when $k$ is increasing. The behavior of this function says that a small number $k<<n$ of drones is enough to produce an admissible isolation time with the random strategy.

\begin{figure}
	\centering
	\includegraphics[width=.6\textwidth]{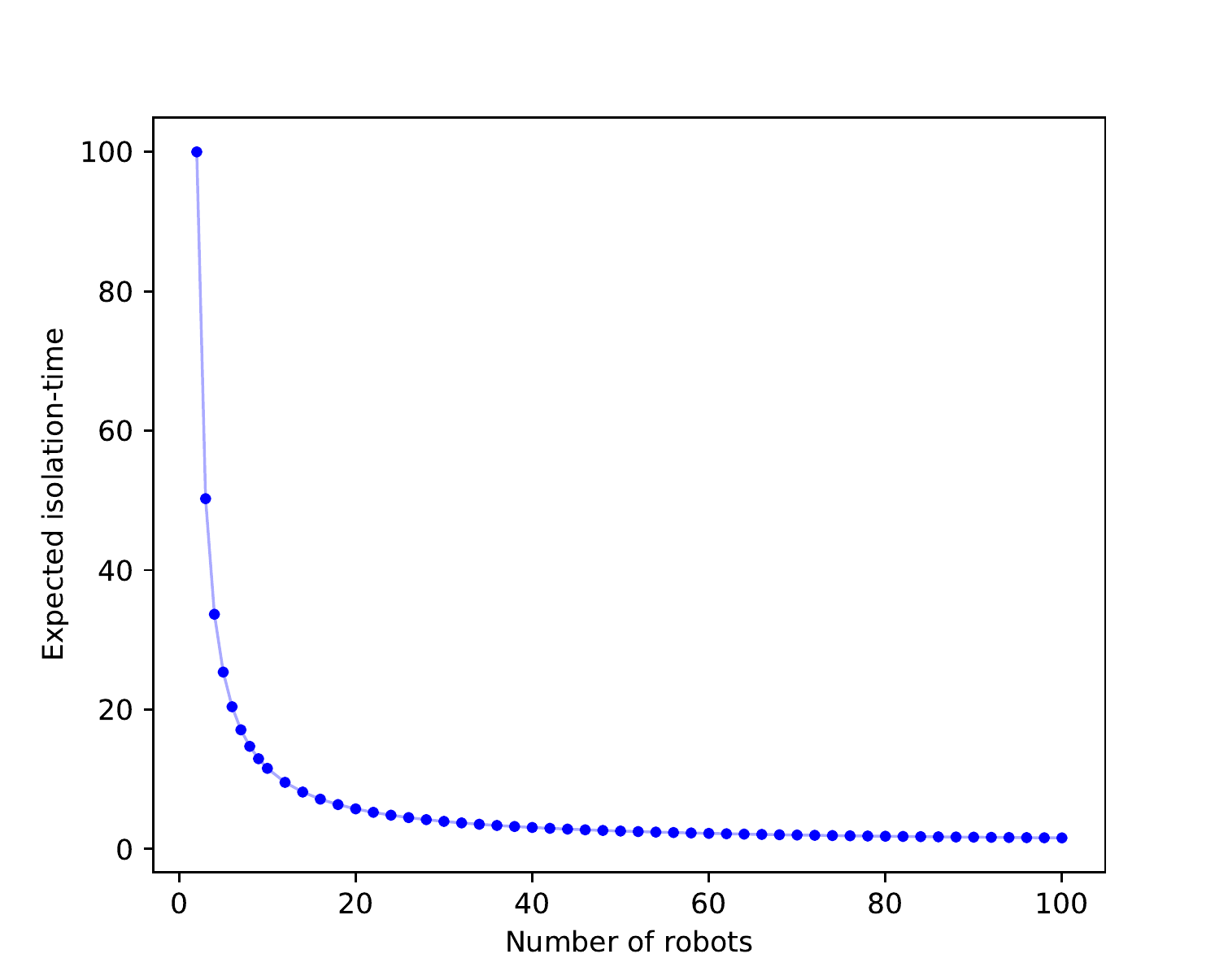}
	\caption{Behavior of the function $I(k)=\left\lceil\frac{n^{k-1}}{n^{k-1}-(n-1)^{k-1}}\right\rceil$, the expected isolation time for  $2\leq k\leq 100$ robots in a scenario of $n=100$ trajectories.}
	\label{fig:theoretical_isolation}
\end{figure}

\section{Experimental results}\label{experiments}
In this section, the validity of the proposed random strategies is evaluated by comparing the values of idle time, isolation time and broadcast time with the results obtained using the deterministic strategy. It is also shown that the obtained bounds for idle time and isolation time are
tight, and then the mixing time is evaluated to explore the rate at which the system converges to the uniform distribution.
These experiments were implemented in Python3.6 using NumPy 1.15.2. The code is available at GitHub\footnote{\url{https://github.com/varocaraballo/random-walks-synch-sq-grid}} for the sake of reproducibility. 

\subsection{The experiments}

Typically, in surveillance tasks with small drones, the map is split into a grid of small cells \cite{acevedo2014one}. Thus, a series of experiments are performed on grid graphs of sizes $10\times 10$, $15\times 15$, $20\times 20$ and $30\times 30$. For each of these $N\times N$ grids,  the random strategies are simulated with $k$ robots, $1\leq k< N^2$. For every grid and each value $k$, 10 repetitions of the experiment are carried out, choosing the starting position of each robot uniformly at random from the synchronized positions. 
Each experiment on an $N\times N$ grid ran for $4N^2$ units of time. In the following, $E_{N,k}^{(i)}$ denotes the $i$-th repetition of an experiment on an $N\times N$ grid using $k$ robots. 

An auxiliary graph is introduced to perform the experiments. Let $\mathcal{F}$ be a synchronized $N\times M$ grid-shaped system. Let $\mathcal{W}_{N,M}$ be a graph, the \emph{walking graph}, whose vertices are the communication links (the edges between neighboring circles of $\mathcal{F}$) and the touching points between the trajectories and an imaginary box circumscribing $\mathcal{F}$; see Figure~\ref{fig:waking_graph}. There is a directed edge $(v,w)$ in $\mathcal{W}_{N,M}$ if there exists an arc of length $\pi/2$ between the points corresponding to $v$ and $w$ following the assigned travel directions. In the following, let $V(\mathcal{W}_{N,M})$ and $E(\mathcal{W}_{N,M})$ denote the set of vertices and edges of $\mathcal{W}_{N,M}$, respectively.  
It is easy to see that $\mathcal{W}_{N,M}$ has $4(N+M)$ vertices and $4NM$ edges. Each robot moves on $\mathcal{F}$ following the travel directions and can perform a shifting operation only at the communication links. Thus, any (not necessarily simple) path in $\mathcal{W}_{N,M}$ corresponds to a valid sequence of movements of a robot in $\mathcal{F}$ where a step in $\mathcal{W}_{N,M}$ corresponds to 1/4 of a time unit in $\mathcal{F}$. In this way, the movement of an agent in $\mathcal{W}_{N,M}$ (following the arcs in $E(\mathcal{W}_{N,M})$) emulates a valid sequence of movements by a robot in $\mathcal{F}$. 
Experiments on the walking graph $\mathcal{W}_{N,N}$ are performed. In the following, the walking graph is denoted by $\mathcal{W}_N$ for easy.

\begin{figure}
	\centering
	\includegraphics[width=.3\textwidth, page =2]{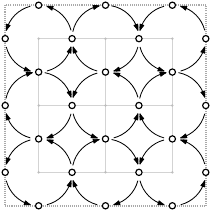}
	\caption{The walking graph of the $3\times 3$ grid SCS shown in Figure~\ref{fig:synchro_sample}(b). The imaginary box where the system is inscribed is indicated by the dotted line. The vertices of the walking graph are the open small circles. Note that the directed arcs between these nodes follow the travel directions assigned in Figure~\ref{fig:synchro_sample}(b).}
	\label{fig:waking_graph}
\end{figure}

\subsection{Idle time experiments}
 \begin{figure*}
	\centering
	\begin{subfigure}{.5\textwidth}
		\centering
		\includegraphics[width=\textwidth]{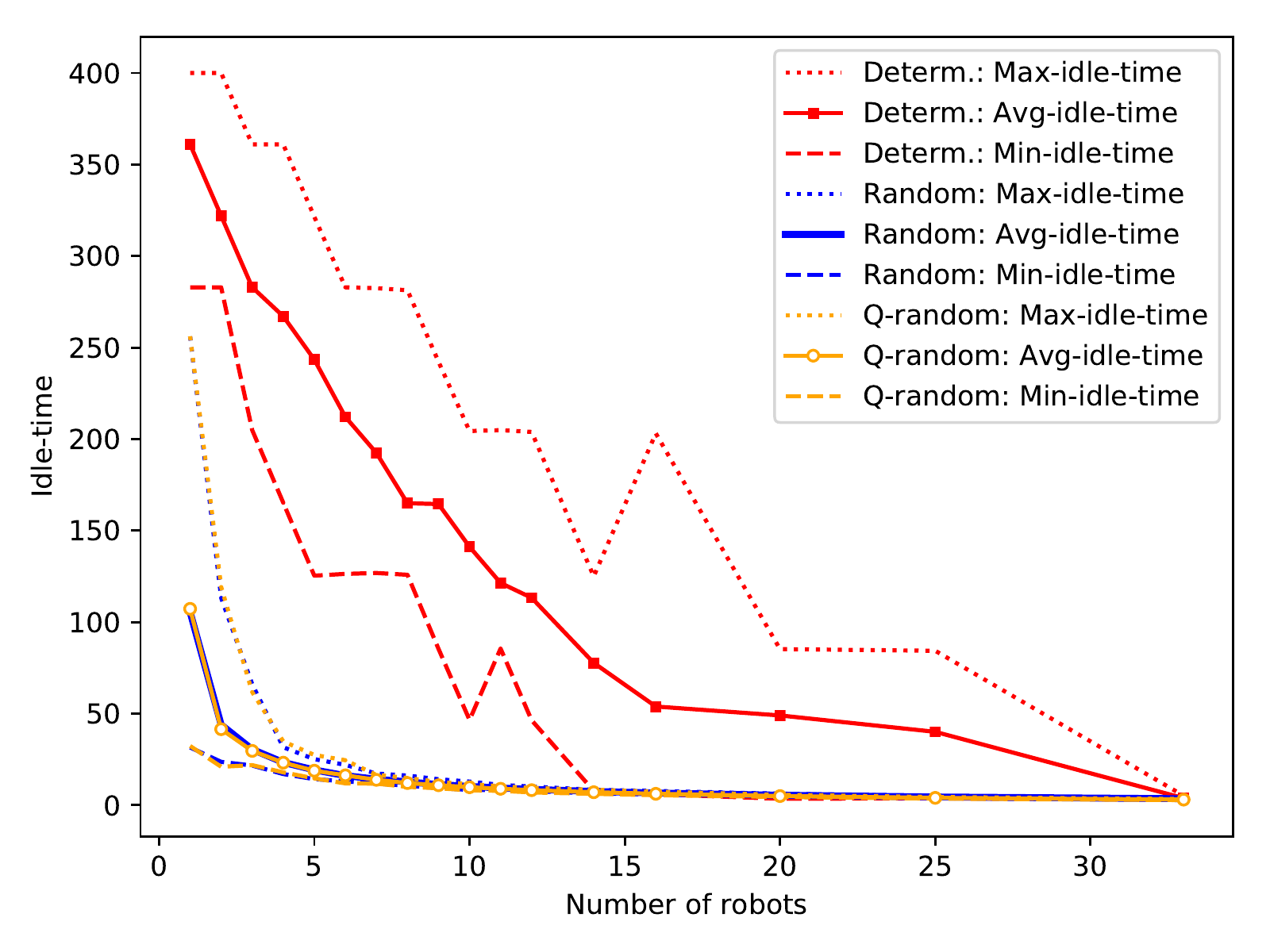}
		\caption{}
	\end{subfigure}%
	\begin{subfigure}{.5\textwidth}
		\centering
		\includegraphics[width=\textwidth]{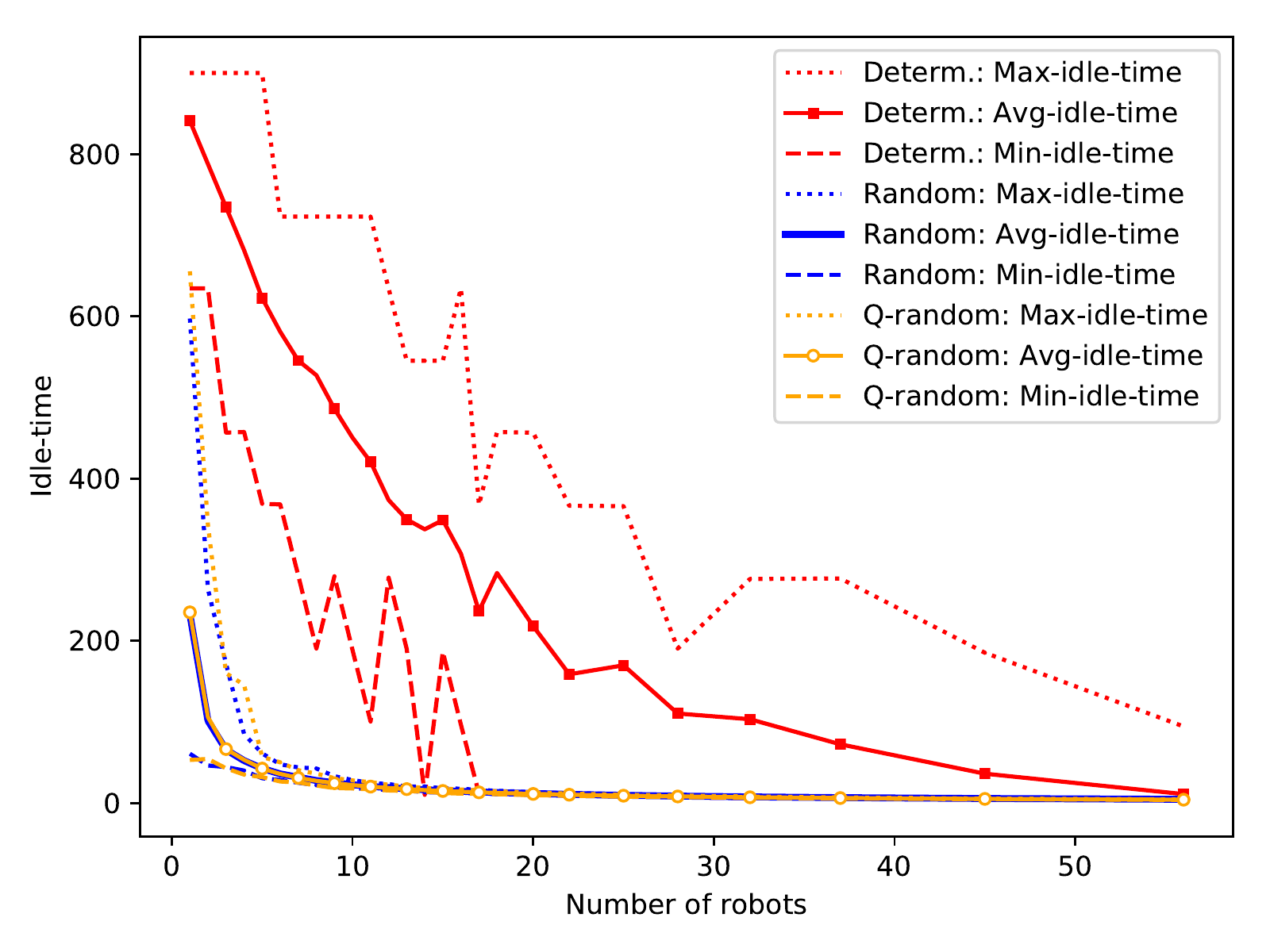}
		\caption{}
	\end{subfigure}\\
	\begin{subfigure}{.5\textwidth}
		\centering
		\includegraphics[width=\textwidth]{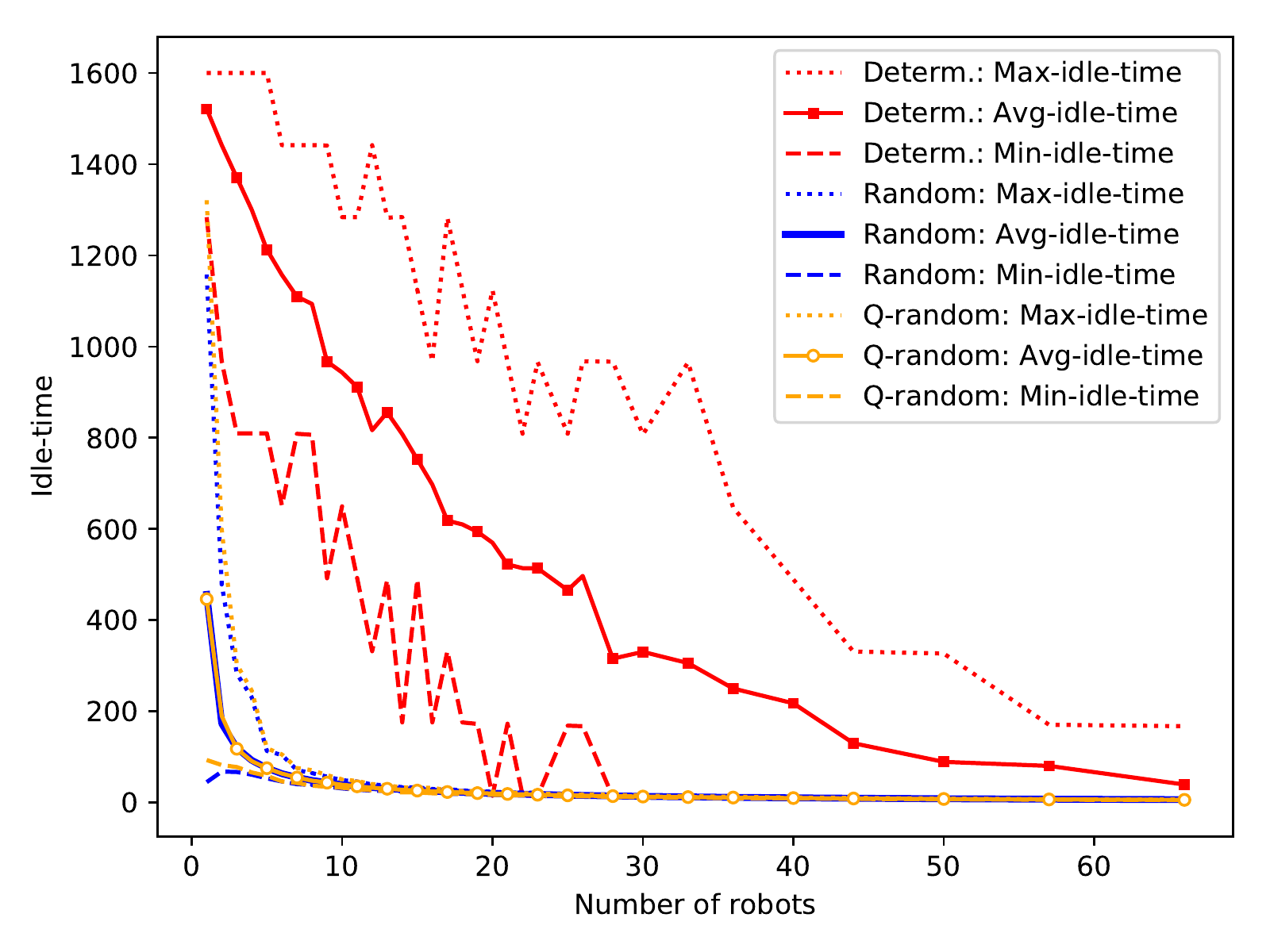}
		\caption{}
	\end{subfigure}%
	\begin{subfigure}{.5\textwidth}
		\centering
		\includegraphics[width=\textwidth]{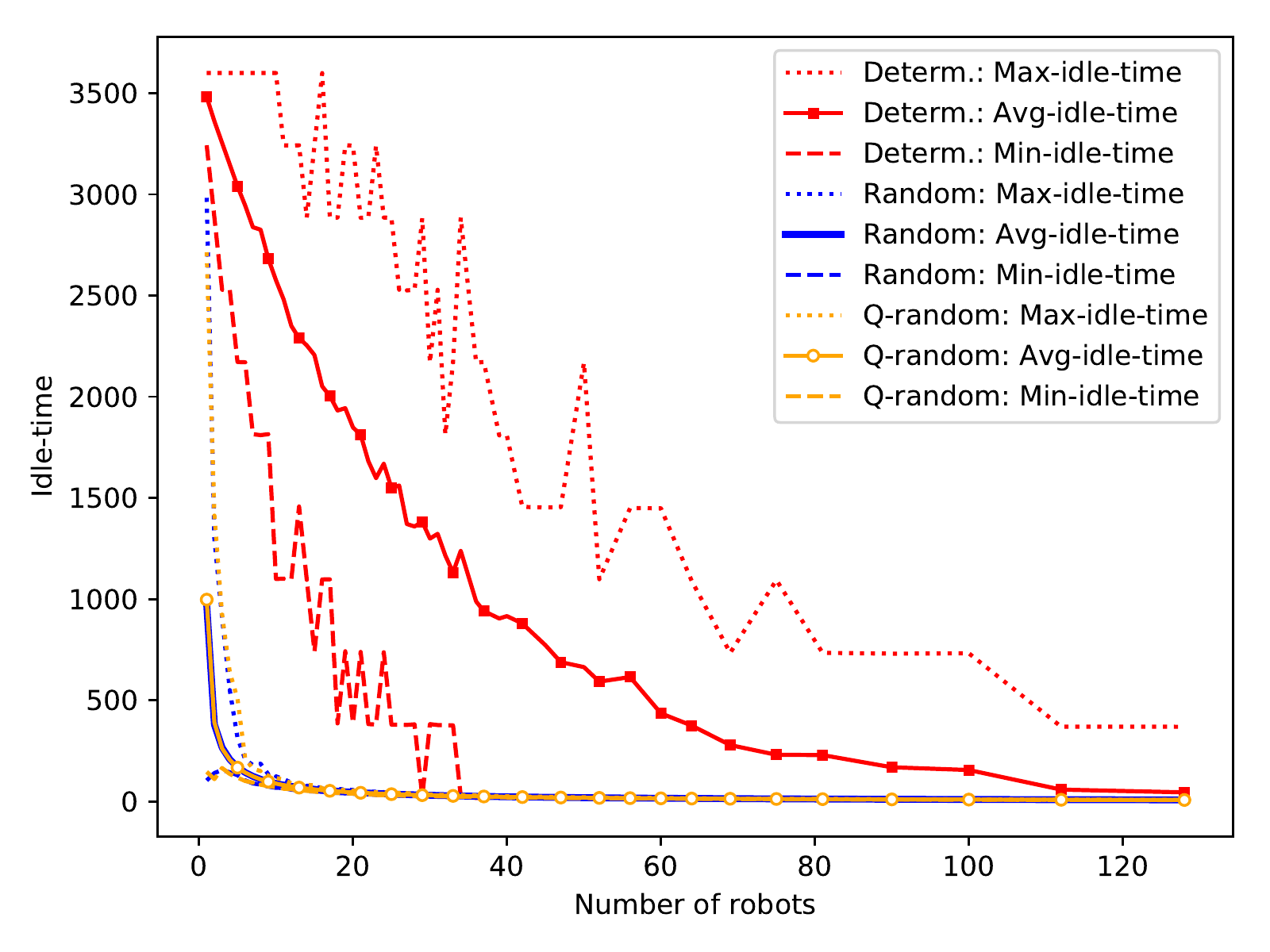}
		\caption{}
	\end{subfigure}\\
	\caption{Comparison of idle time obtained in the experiments using the random, quasi-random (Q-random) and deterministic strategies: (a) $10\times 10$ grid SCS, (b) $15\times 15$ grid SCS, (c) $20\times 20$ grid SCS and (d) $30\times 30$ grid SCS.}
	\label{fig:experimental_idle}
\end{figure*}
Note that every point on a circle of an $N\times N$ grid SCS is on some edge of $\mathcal{W}_N$. Hence, the idle time of an $N\times N$ grid SCS using $k$ robots is the idle time of the edges of $\mathcal{W}_N$ using $k$ robots. In order to measure this value, we do the following: in each experiment $E_{N,k}^{(i)}$, for every arc $e\in E(\mathcal{W}_N)$, we count the number of times that $e$ is traversed by a robot and when it is traversed. Having all the visits to an edge $e$, we can compute the average time $\texttt{idle}_e^{(i)}(k)$ between consecutive passes through edge $e$ in the experiment $E_{N,k}^{(i)}$.

Let $\texttt{idle}_e(k) = \frac{1}{10}\sum_{i=1}^{10}\texttt{idle}_e^{(i)}(k)$ denote the average time between consecutive passes through edge $e$ over all the 10 repetitions of the experiment using $k$ robots. Finally, in order to compare the idle time of these strategies, taking into account the worst case, average case and best case, we consider the functions
\begin{align*}
\displaystyle\texttt{max\_idle}_{\mathcal{W}_N}(k) &=\max_{e\in E(\mathcal{W}_N)}\{\texttt{idle}_e(k)\},\\
\displaystyle\texttt{avg\_idle}_{\mathcal{W}_N}(k) &= \frac{1}{|E(\mathcal{W}_N)|}\sum_{e\in E(\mathcal{W}_N)}\texttt{idle}_e(k) \text{,and}\\
\displaystyle\texttt{min\_idle}_{\mathcal{W}_N}(k) &= \min_{e\in E(\mathcal{W}_N)}\{\texttt{idle}_e(k)\}, \text{ respectively.}
\end{align*}

These three functions are computed using the three strategies and the results are shown in Figure~\ref{fig:experimental_idle}. The three functions $\texttt{max\_idle}$, $\texttt{avg\_idle}$ and $\texttt{min\_idle}$ are shown using dotted, solid and dashed lines, respectively. The functions for the
random strategy is shown in blue, for
the quasi-random strategy in orange,
and for the deterministic strategy in
red .


Figures \ref{fig:experimental_idle}(a),
(b), (c) and (d) show the results for
grid SCSs of sizes $10\times 10$,
$15\times 15$, $20\times 20$ and
$30\times 30$, respectively. Each
subfigure shows the evolution of
$\texttt{max\_idle}$,
$\texttt{avg\_idle}$ and
$\texttt{min\_idle}$ for the three
strategies. Note that the graphs of
Figure~\ref{fig:experimental_idle} show the behavior of the
functions until the maximum tested
value of $k $ (which is $N^2$) in order to emphasize the differences
between the three tested strategies,
and these differences are most notable
for low values of $k $ (after a certain
value of $k $, the functions are
similar for all the strategies). Note
that for low values of $k $, there is a
huge difference between the
deterministic strategy and the other
strategies. This difference is not
surprising because we know from
Bereg et al.~\cite{bereg2020robustness} that an
$N\times N$ grid SCS graph has $N$
\emph{rings}\footnote{ A ring in a SCS is the locus of points visited by a starving robot following the assigned movement direction in each trajectory and always shifting to the neighboring trajectory at the corresponding link positions.}, and to cover everything
using the deterministic strategy, at
least one robot per ring is required.
Therefore, with fewer than $N$ robots
it is not possible to cover every
point and a very high idle
time is obtained. In these cases, the idle time is
bounded by the duration of the
simulation ($4N^2$), although what
actually happens is that some points
are never visited. Note that because
the starting positions are assigned
randomly, even if $k \geq N$, some
rings could be empty of robots.
Therefore, all the points involved in
these rings are uncovered.

Note also that the $\texttt{max\_idle}$ is not $4N^2$ for every $k< N$. This is because the value $\texttt{idle}_e(k)$ is amortized among the 10 repetitions of the experiment (in some experiment $E_{N,k}^{(i)}$ a ring $r$ may be empty of robots and then $\texttt{idle}_e^{(i)}(k) = 4N^2$ for every edge $e$ involved in $r$, but in some other experiment $E_{N,k}^{(j)}$, the ring $r$ may have $c>0$ robots, thus $\texttt{idle}_e^{(j)}(k) = N/c$ for every edge $e$ involved in $r$). 

A very similar behavior can be observed
with the random and quasi-random
strategies. In fact, the function
$\texttt{avg\_idle}_{\mathcal{W}_N}$
is close to the function $f(k ) =
N^2/k $ and this is the best possible
behavior, taking into account that there is
a trajectory of length $2\pi N^2$
(the sum of the $N^2$ circles) to
cover and $k $ robots moving at $2\pi$
length units per time unit. Moreover,
when one of these random strategies is
used, the function
$\texttt{avg\_idle}_{\mathcal{W}_N}$
starts decreasing rapidly before $k =N=\sqrt{n}$ (recall $n=N^2$
in these grid SCSs), and then the rate of
decrease is notably reduced. This
behavior indicates that when one of
the random strategies is being used in
an SCS with $k \geq N$ robots, the
addition of more robots to the system
does not bring about a significant
improvement in idle time. Another
conclusion that can be drawn from the
experiments is that when either of the
two random strategies is used with few
robots ($k \approx N$), the idle time
is really good. However, to obtain similar
results using the deterministic
strategy the system needs, with high
probability, a much larger number of
robots.




\subsection{Isolation time experiments}

\begin{figure*}
	\centering
	\begin{subfigure}{.5\textwidth}
		\centering
		\includegraphics[width=\columnwidth]{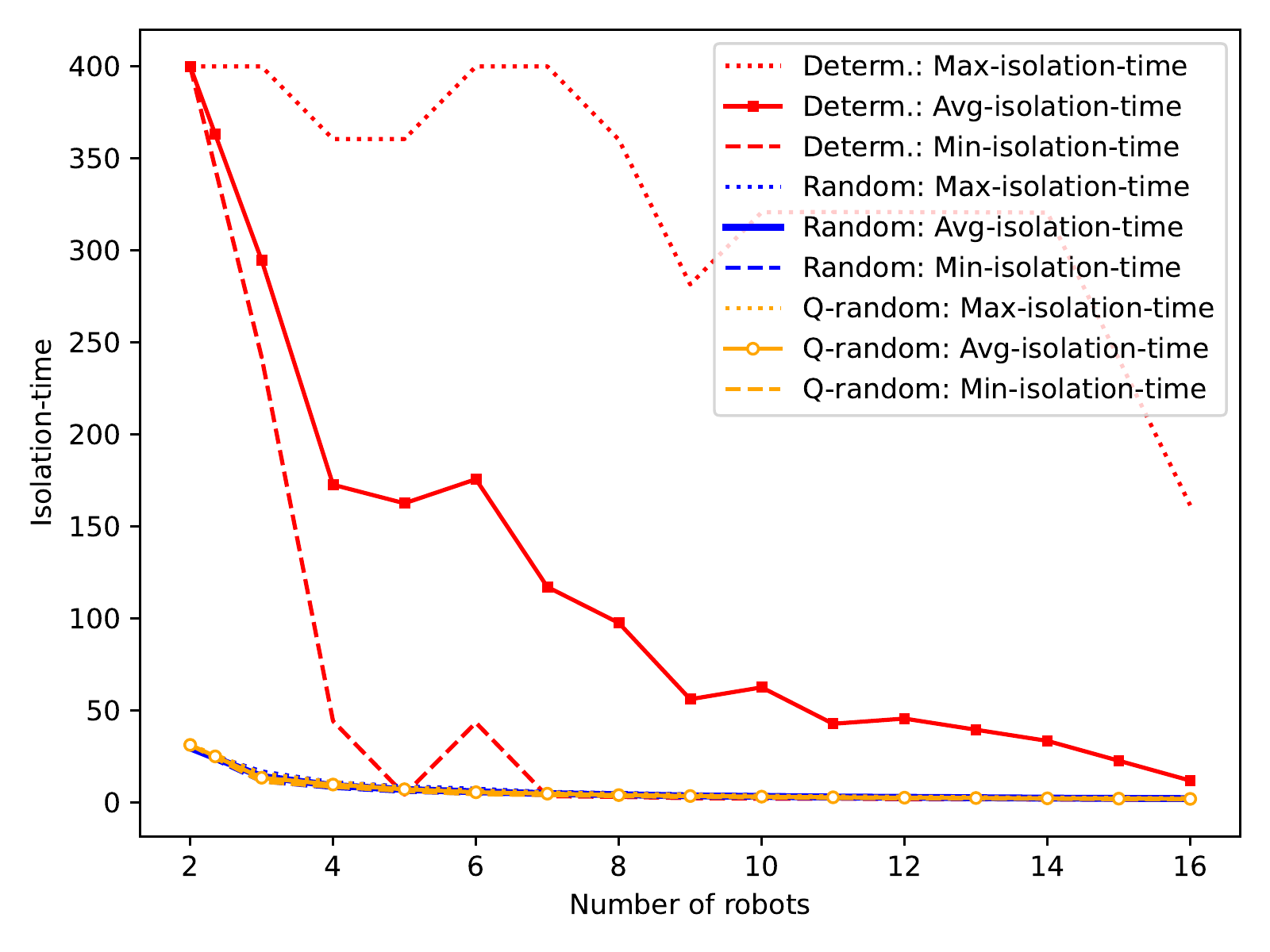}
		\caption{}
	\end{subfigure}%
	\begin{subfigure}{.5\textwidth}
		\centering
		\includegraphics[width=\columnwidth]{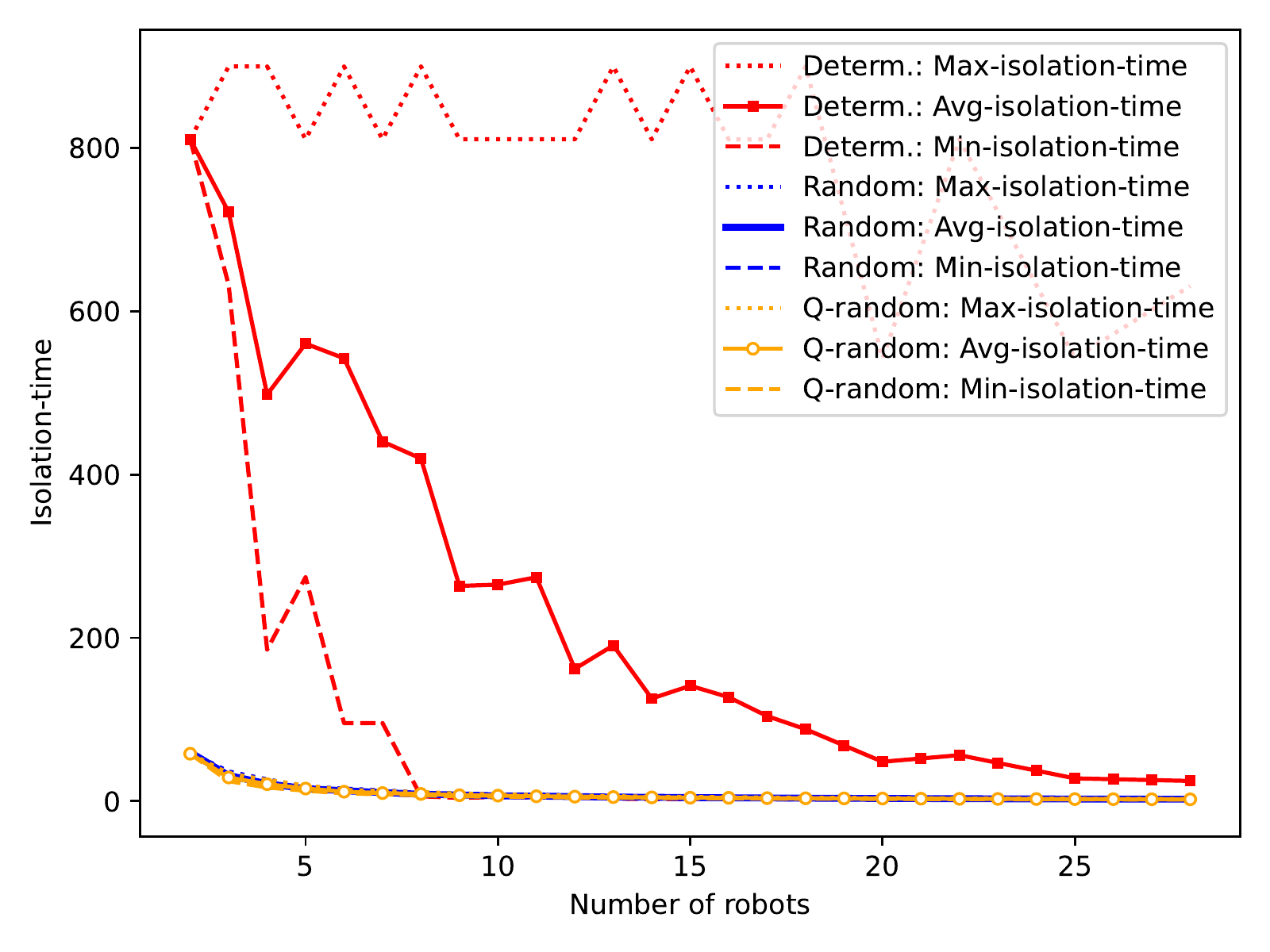}
		\caption{}
	\end{subfigure}\\
	\begin{subfigure}{.5\textwidth}
		\centering
		\includegraphics[width=\columnwidth]{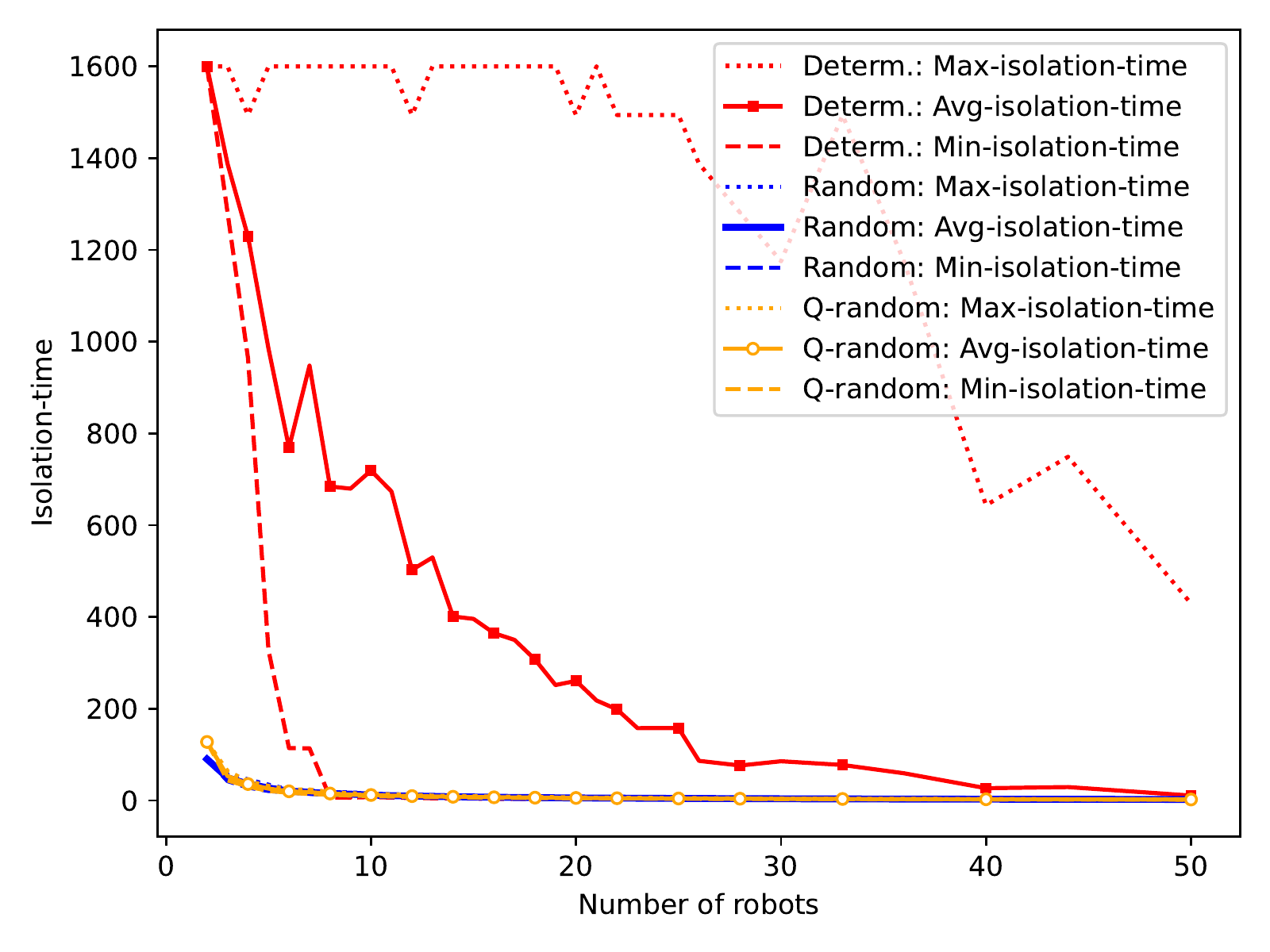}
		\caption{}
	\end{subfigure}%
	\begin{subfigure}{.5\textwidth}
		\centering
		\includegraphics[width=\columnwidth]{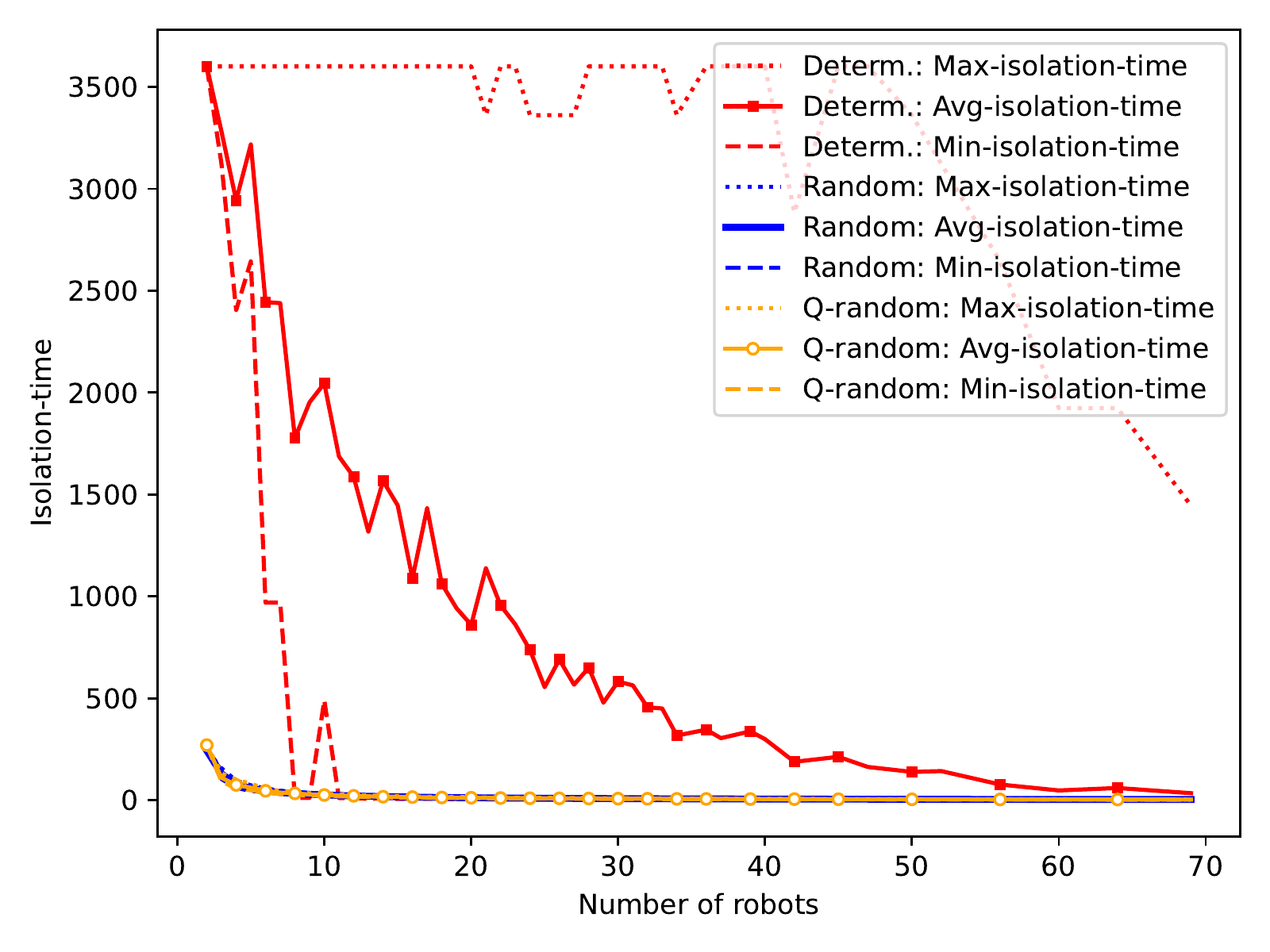}
		\caption{}
	\end{subfigure}\\
	\caption{Comparison of the isolation time obtained in the experiments using the random, quasi-random (Q-random) and deterministic strategies: (a) $10\times 10$ grid SCS, (b) $15\times 15$ grid SCS, (c) $20\times 20$ grid SCS and (d) $30\times 30$ grid SCS.}
	\label{fig:experimental_isolation}
\end{figure*}

Note that every meeting between two robots in an $N\times N$ grid SCS takes place at one of the vertices of the walking graph $\mathcal{W}_N$. Hence, in an $N\times N$ grid SCS using a set $U$ of $k$ robots, the isolation time of the system is the expected time between two consecutive meetings at a vertex of $\mathcal{W}_N$ using $k$ robots. In order to measure this value, we do the following: in each experiment $E_{N,k}^{(i)}$, for every robot $u\in U$, we count the number of times that $u$ meets some other robot and when these meetings occur. Having all the meetings of $u$ we can compute the average time $\texttt{isolation}_u^{(i)}(k)$ between consecutive meetings of $u$ in the experiment $E_{N,k}^{(i)}$. 

In order to compare the isolation time of the strategies, tacking into account the worst case, average case and best case, we define the functions:
\begin{align*}
\displaystyle\texttt{max\_isolation}^{(i)}_{\mathcal{W}_N}(k) &=\max_{u\in U}\{\texttt{isolation}^{(i)}_u(k)\},\\
\displaystyle\texttt{avg\_isolation}^{(i)}_{\mathcal{W}_N}(k) &= \frac{1}{k}\sum_{u\in U}\texttt{isolation}^{(i)}_u(k),\\
\displaystyle\texttt{min\_isolation}^{(i)}_{\mathcal{W}_N}(k) &= \min_{u\in U}\{\texttt{isolation}^{(i)}_u(k)\}, 
\end{align*}
respectively. 

Then we compute the averaged values over all 10 repetitions of the experiments:
\begin{align*}
\displaystyle\texttt{max\_isolation}_{\mathcal{W}_N}(k) &=
\sum_{i=1}^{10}\frac{\texttt{max\_isolation}^{(i)}_{\mathcal{W}_N}(k)}{10},\\
\displaystyle\texttt{avg\_isolation}_{\mathcal{W}_N}(k) &=\sum_{i=1}^{10}\frac{\texttt{avg\_isolation}^{(i)}_{\mathcal{W}_N}(k)}{10},\\
\displaystyle\texttt{min\_isolation}_{\mathcal{W}_N}(k) &= \sum_{i=1}^{10}\frac{\texttt{min\_isolation}^{(i)}_{\mathcal{W}_N}(k)}{10}.
\end{align*}

\begin{figure*}
	\centering
	\begin{subfigure}{.5\textwidth}
		\centering
		\includegraphics[width=\columnwidth]{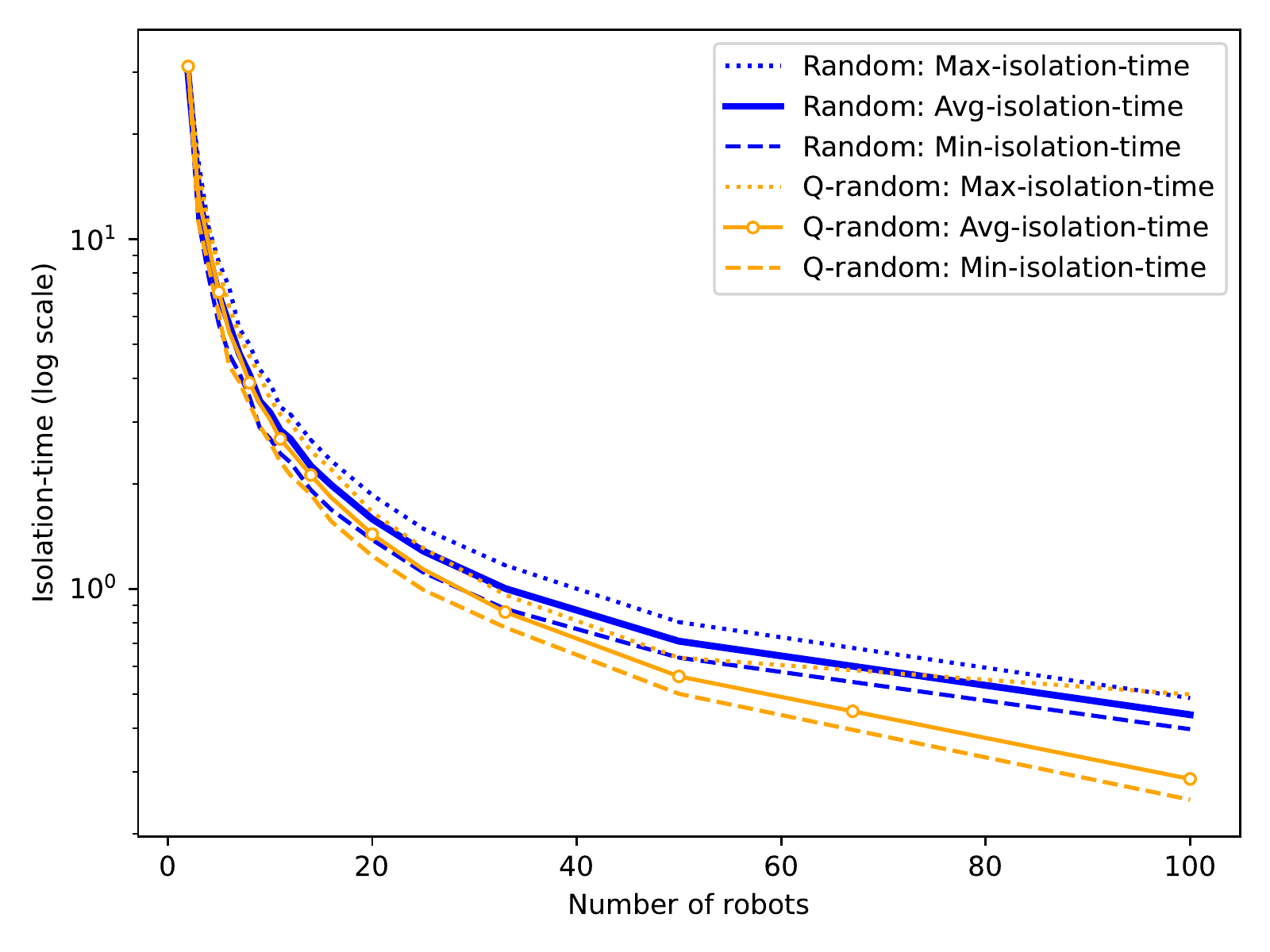}
		\caption{}
	\end{subfigure}%
	\begin{subfigure}{.5\textwidth}
		\centering
		\includegraphics[width=\columnwidth]{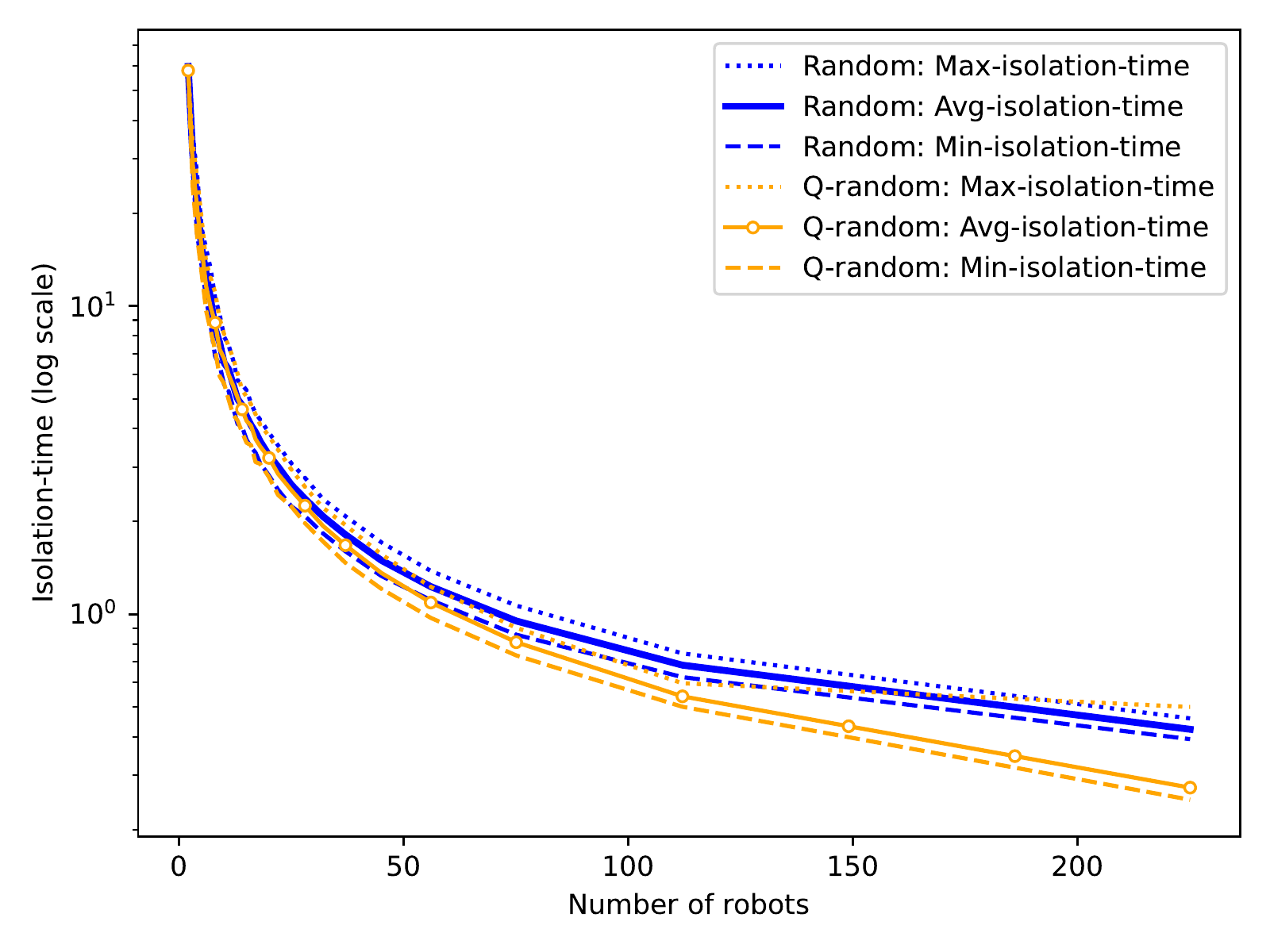}
		\caption{}
	\end{subfigure}\\
	\begin{subfigure}{.5\textwidth}
		\centering
		\includegraphics[width=\columnwidth]{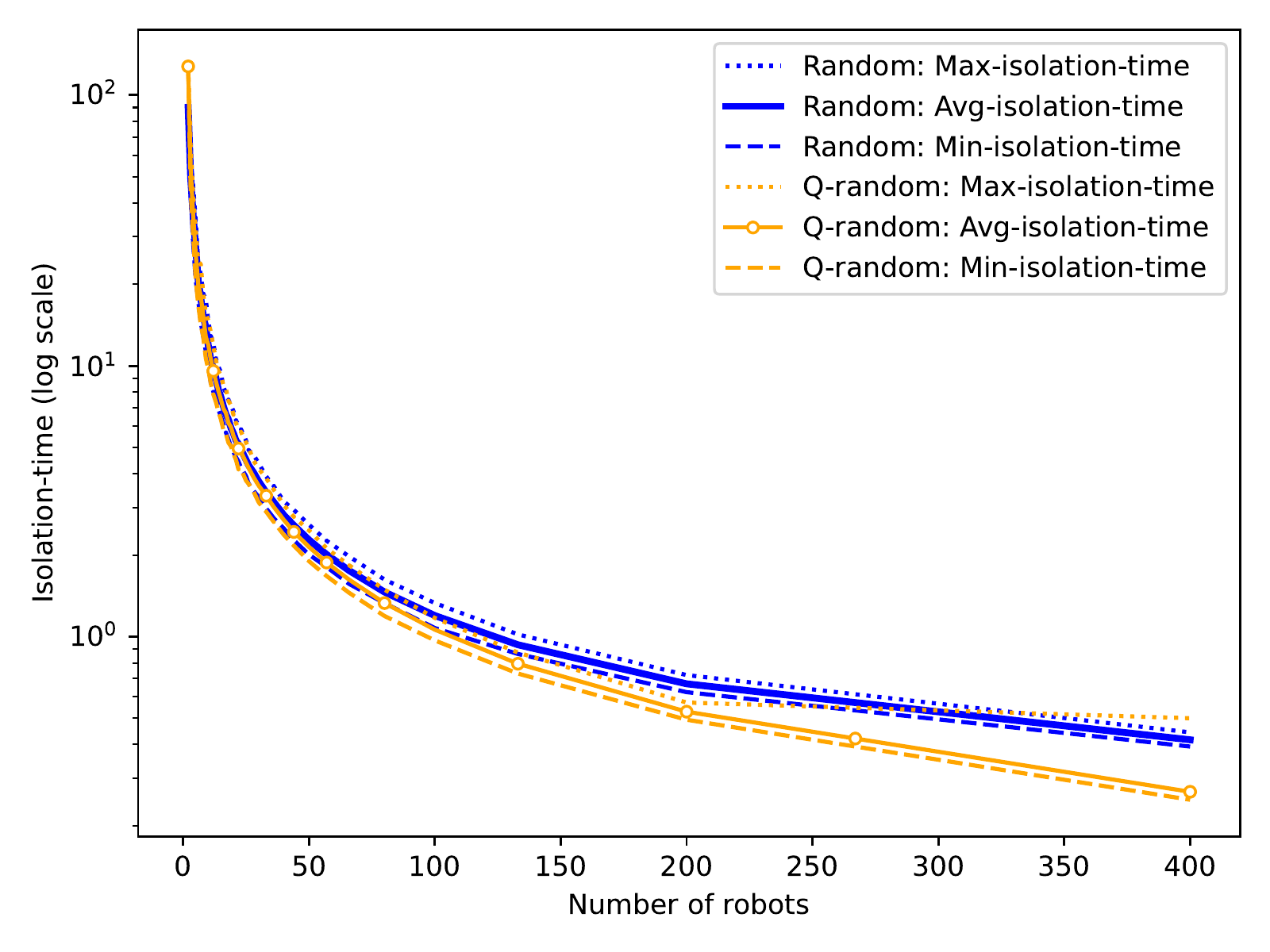}
		\caption{}
	\end{subfigure}%
	\begin{subfigure}{.5\textwidth}
		\centering
		\includegraphics[width=\columnwidth]{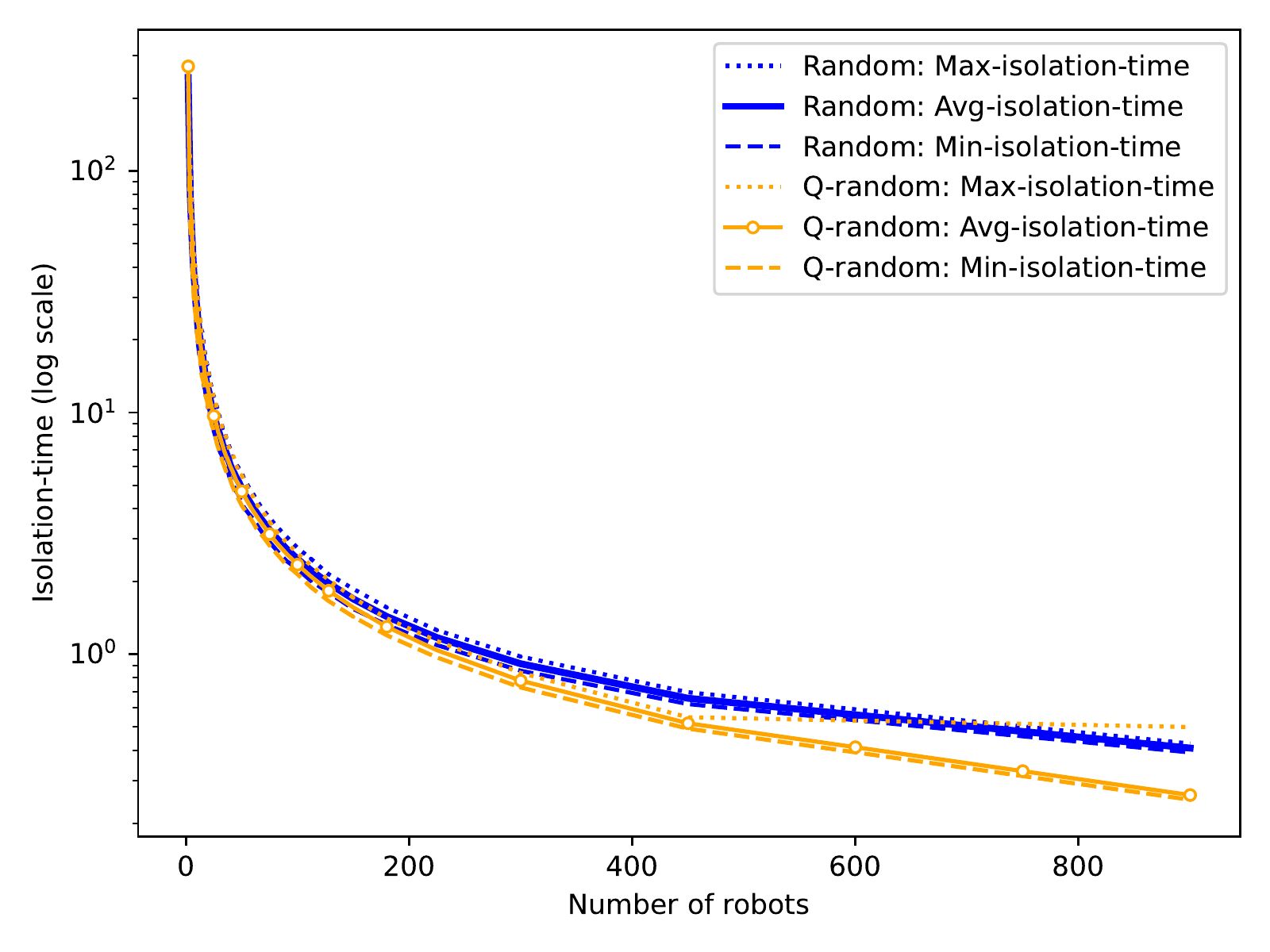}
		\caption{}
	\end{subfigure}\\
	\caption{Comparison of the isolation time obtained in the experiments using the random and quasi-random (Q-random) strategies: (a) $10\times 10$ grid SCS, (b) $15\times 15$ grid SCS, (c) $20\times 20$ grid SCS and (d) $30\times 30$ grid SCS.}
	\label{fig:experimental_isolation_rand}
\end{figure*}

We computed the values of these three global functions using the three strategies and the results are shown in Figure~\ref{fig:experimental_isolation}. The three functions
$\texttt{max\_isolation}$,
$\texttt{avg\_isolation}$ and
$\texttt{min\_isolation}$ are shown
using dotted, solid and dashed 
lines, respectively. The functions for the
random strategy is shown in blue, for
the quasi-random strategy in orange,
and for the deterministic strategy in red.

Figures
\ref{fig:experimental_isolation}(a),
(b), (c) and (d) show the results on
grid SCSs of sizes $10\times 10$,
$15\times 15$, $20\times 20$ and
$30\times 30$, respectively. In each
subfigure, the evolution of
$\texttt{max\_isolation}$,
$\texttt{avg\_isolation}$ and
$\texttt{min\_isolation}$ for the
three strategies is shown for $k \geq 2$ (note
that a single robot in the system is
always isolated).
As in the previous subsection,
Figure~\ref{fig:experimental_isolation} does not show the behavior of the
functions until the maximum tested
value of $k $ because the differences
between the strategies are much more
evident for low values of $k $ (after a
certain value of $k $, the functions
are similar for all three strategies).
Note that for low values of $k $, the
deterministic strategy is much worse
than the random strategies. Note also
that for two robots, the values of
$\texttt{max\_isolation}$,
$\texttt{avg\_isolation}$ and
$\texttt{min\_isolation}$ are the same
(that is, the three functions start at
the same point for each strategy).
This behavior is because when there
are only two robots, if one robot has
a meeting, it is with the other one,
so they both have exactly the same
statistics. Therefore, their average,
maximum and minimum isolation times
have the same value.

From Bereg et al.~\cite{bereg2020robustness}, it is known that in an $N\times N$ grid
SCS with robots using the
deterministic strategy, two robots
meet each other if and only if they
start in circles that are in the same
row or column. Thus, if there are
$k \ll N^2$ robots, it is very probable
that a robot would start in a circle
that is in a different row and column
from those of any other robot in the
system. Because of this, the values of
$\texttt{max\_isolation}$ are very
high for small values of $k $. Also,
note that using a relatively small
number of robots, but more than two,
the probability of at least two of
them start in the same column or
the same row is high. Because of this,
the $\texttt{min\_isolation}$ function
decreases quickly. Therefore, using
the deterministic strategy, there is a
huge difference between the functions
$\texttt{max\_isolation}$,
$\texttt{avg\_isolation}$ and
$\texttt{min\_isolation}$. Finally,
note that $\texttt{avg\_isolation}$
is closer to $\texttt{min\_isolation}$
than to $\texttt{max\_isolation}$.
This indicates that when the group of
robots is small, there is a high
probability that one of them may be
isolated, but it is very probable that
the number of non-isolated robots will
be greater than the number of isolated
robots.

In
Figure~\ref{fig:experimental_isolation}, it is difficult to observe the
behavior of the random strategies due
to the high values of isolation time
when the deterministic strategy is
used. Note that, very similar
behavior is obtained using either of the
randomized strategies, being both much better than the deterministic
strategy. 
Figure~\ref{fig:experimental_isolation_rand} compares the values of isolation
time using the random strategies on a
logarithmic scale for the isolation
time axis. Note that the greater the
number of robots, the better the
quasi-random strategy is (in average) with respect
to the random strategy. Another conclusion that can be extracted from
Figure~\ref{fig:experimental_isolation_rand} is that using any of the random
strategies, small isolation time
values are obtained with very few robots
compared to the number of trajectories
(cells); i.e. $k \approx N=\sqrt{n}$. However, to get similar results
using the deterministic strategy, a
much larger number of robots are required, with \emph{high probability} (see
Figure~\ref{fig:experimental_isolation}). Moreover, the value $k =N$ marks a
threshold in the function
$\texttt{avg\_isolation}$ when one of
the two random strategies is being
used, because the addition of more
robots to an SCS with $k \geq N$ robots
does not represent a significant
reduction of the isolation time.

\subsection{Broadcast time experiments}\label{sec:broadcast-exp}
\begin{figure*}
	\centering
	\begin{subfigure}{.5\textwidth}
		\centering
		\includegraphics[width=\columnwidth]{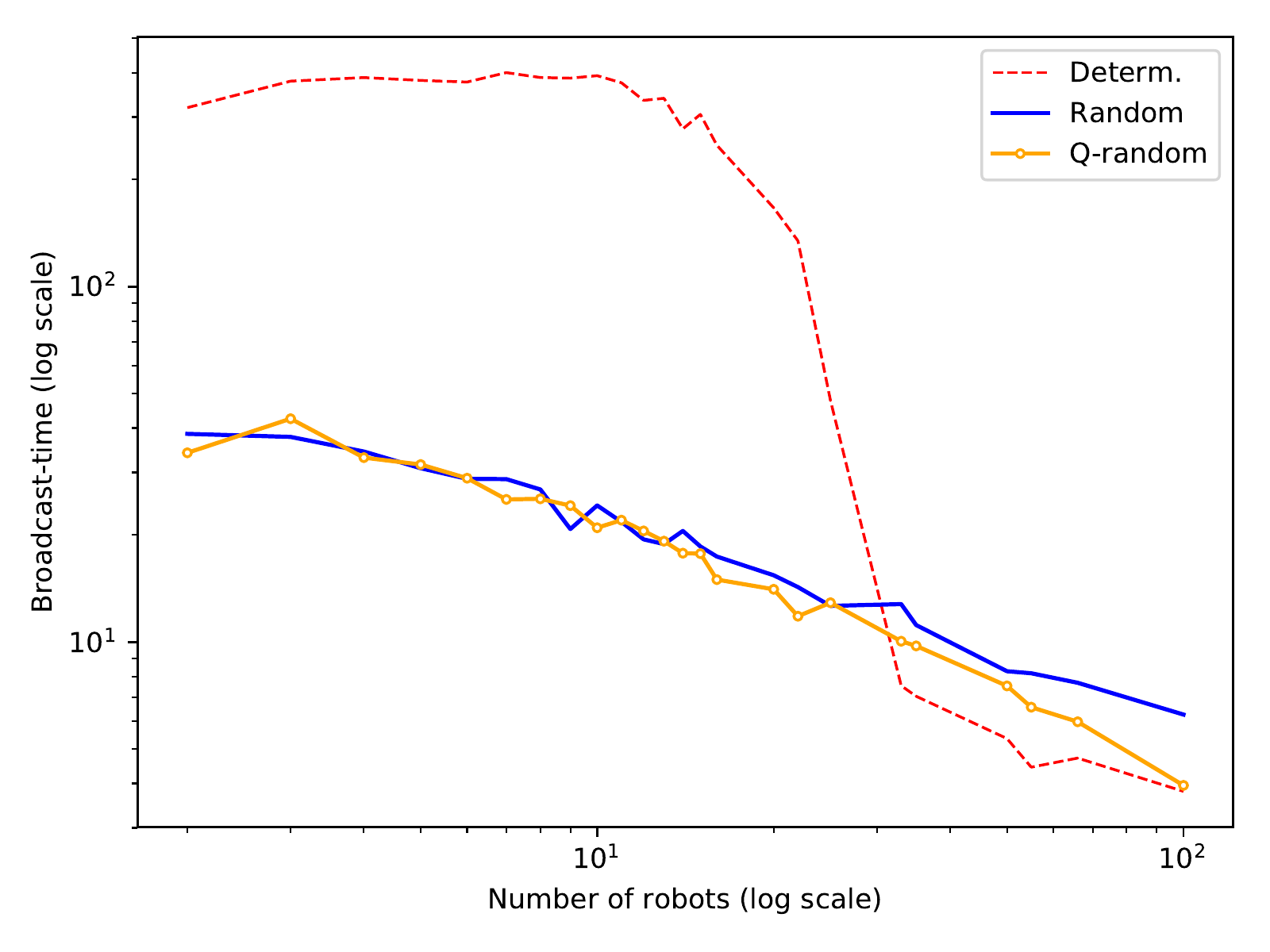}
		\caption{}
	\end{subfigure}%
	\begin{subfigure}{.5\textwidth}
		\centering
		\includegraphics[width=\columnwidth]{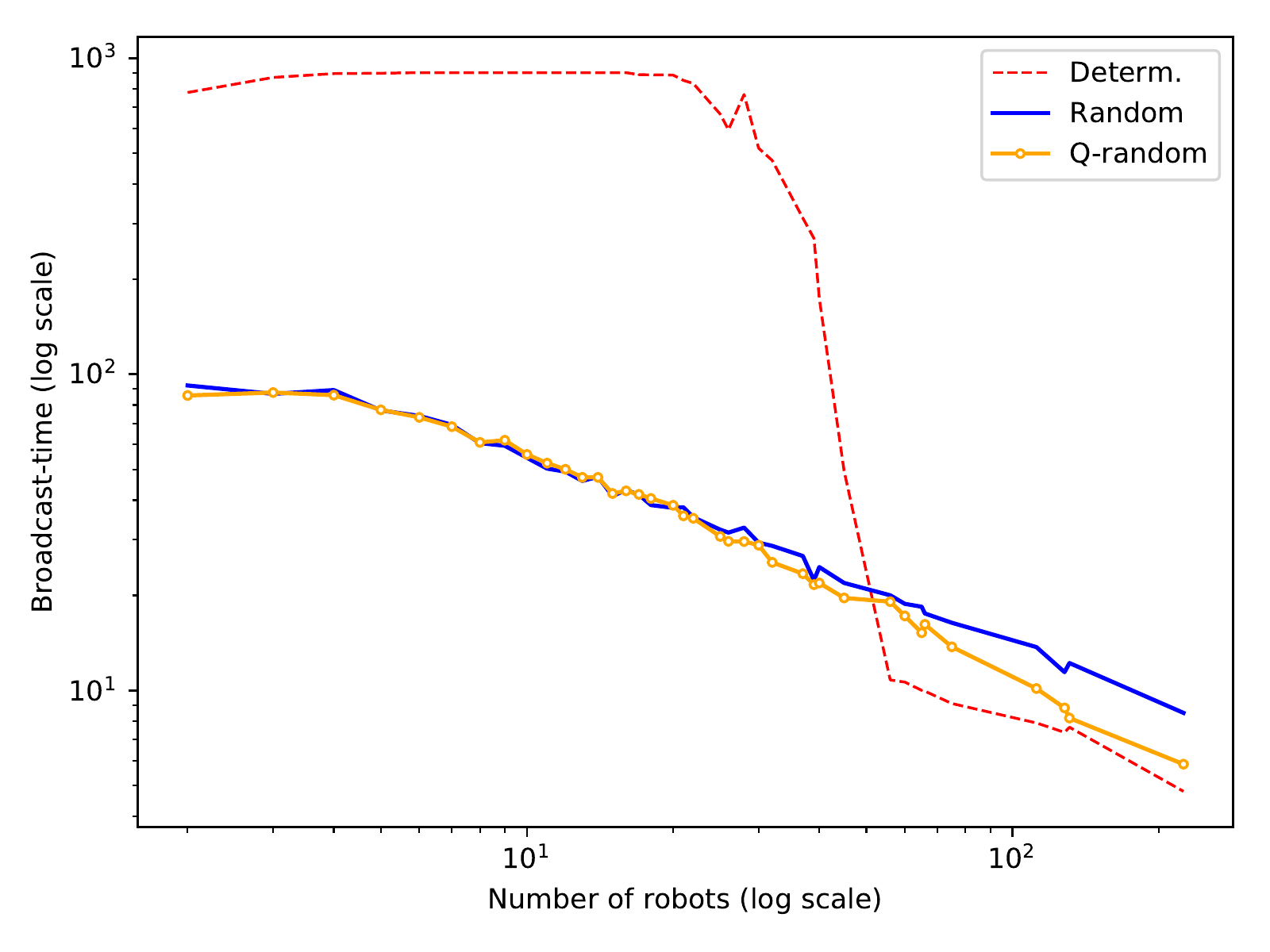}
		\caption{}
	\end{subfigure}\\
	\begin{subfigure}{.5\textwidth}
		\centering
		\includegraphics[width=\columnwidth]{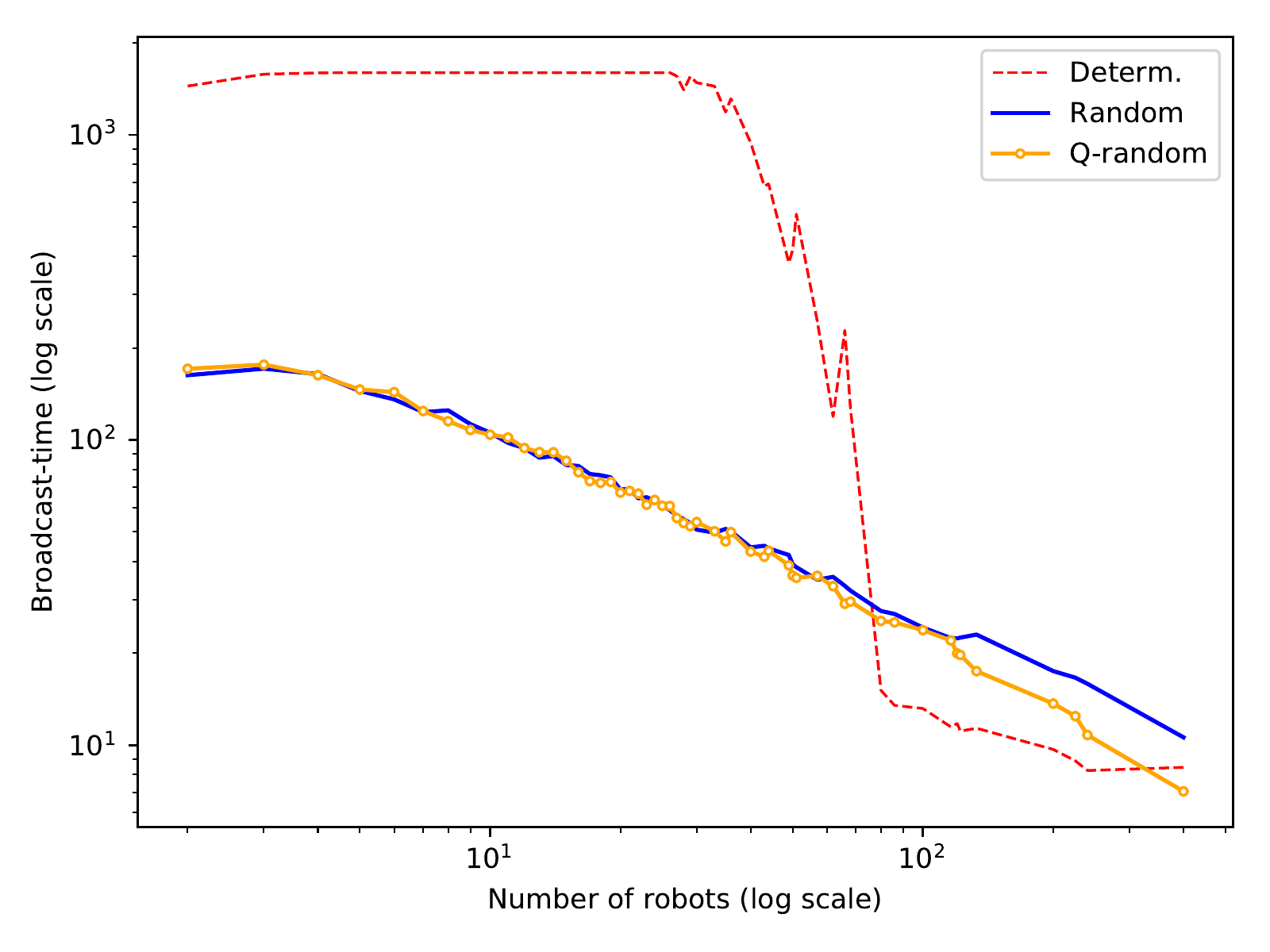}
		\caption{}
	\end{subfigure}%
	\begin{subfigure}{.5\textwidth}
		\centering
		\includegraphics[width=\columnwidth]{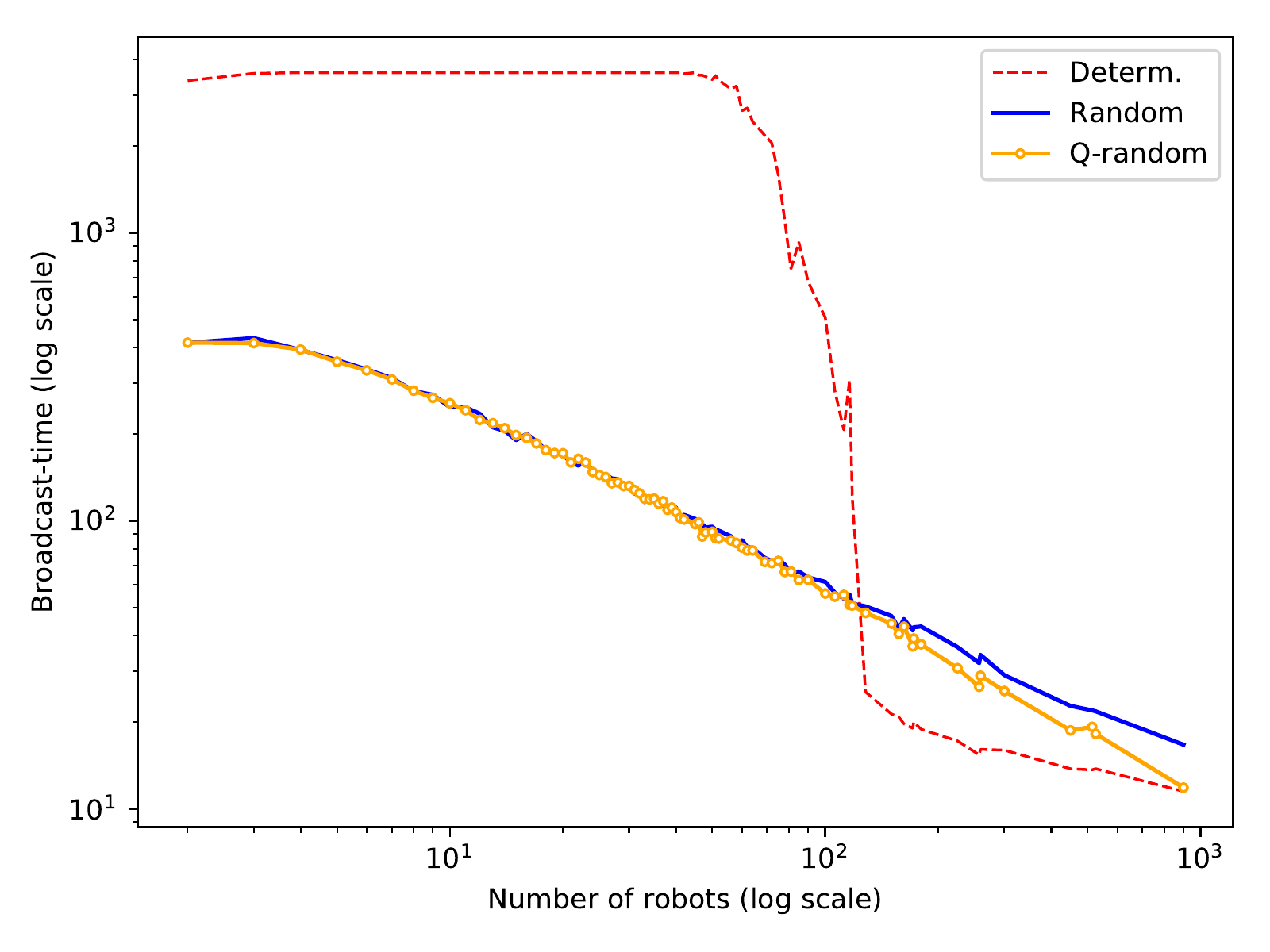}
		\caption{}
	\end{subfigure}\\
	\caption{Comparison of the broadcast time obtained in the experiments using the random, quasi-random (Q-random) and deterministic strategies: (a) $10\times 10$ grid SCS, (b) $15\times 15$ grid SCS, (c) $20\times 20$ grid SCS and (d) $30\times 30$ grid SCS.}
	\label{fig:experimental_broadcast}
\end{figure*}

In this subsection, the broadcast time measure it is studied under the three different proposed strategies for grid-type SCSs. To evaluate this measure, a different experiment was designed. Suppose that we have an $N\times N$ grid SCS $\mathcal{F}$, a value $k\geq 2\, (k \in \mathbb{N})$ and a strategy $\sigma$. In order to estimate the broadcast time in $\mathcal{F}$ using $k$ robots applying the strategy $\sigma$, we do the following: we randomly set the $k$ starting positions of each robot with uniform probability in the walking graph $\mathcal{W}_N$ of $\mathcal{F}$. Then we randomly choose a robot to emit a message. After that, the simulation starts, using the strategy $\sigma$ until all robots in the team have received the message or $4N^2$ time units have elapsed. Let
$E_{N,k }^{(i)}$ denote the $i$-th
repetition of this experiment and let
$\texttt{broadcast}_{\mathcal{W}_N}^{(
i)}(k )$ denote its simulation time.
Note that
$\texttt{broadcast}_{\mathcal{W}_N}^{(
i)}(k )=\min\{4N^2, t\}$ where $t$ is
the time required for the message to
be known to all the robots in the
team. We repeat this experiment $N^2$
times and estimate the broadcast time
of $\mathcal{F}$ using $k $ robots by
the following function:
\[\displaystyle\texttt{broadcast}_{\mathcal{W}_N}(k )=\frac{1}{N^2}\sum_{i=1}^{
N^2}\texttt{broadcast}_{\mathcal{W}_N}^
{(i)}(k ).\]

We computed the values of the function $\texttt{broadcast}_{\mathcal{W}_N}$ for $2\leq k< N^2$ and $N=10$, 15, 20 and 30. The results of our experiments are shown in
Figure~\ref{fig:experimental_broadcast}, using a logarithmic scale for both the broadcast time and the number of robots used. 
According to this picture, for small values of $k$, the random strategies yield much better results than the deterministic one. 
As mentioned above, two robots using the deterministic strategy meet each other only if their initial positions are in the same row or the same column~\cite{bereg2020robustness}. Now consider a new auxiliary graph $H$ as follows: add a vertex for every robot and there is an edge between two robots if they are on circles that are in the same row or the same column. Obviously, it is possible to make a broadcast if and only if  $H$ is connected. When the number of robots is small, it is very probable than $H$ is not connected. For this reason, broadcast times are very high when the deterministic strategy is used with a small number of robots. 

When the number of robots is sufficiently high ($k>N^{4/3}$
), the three strategies are more similar. However, it can be observed that the deterministic strategy is better than the others, although the quasi-random strategy gives similar results; see Figure~\ref{fig:experimental_broadcast}. 
The reason for this behavior is that when $k$ is large, the auxiliary graph $H$ is connected and, a broadcast can be completed in approximately $N$ time units. Also, note that, when the quasi-random strategy is used, the more robots the system has, the more similar this strategy is to the deterministic strategy.


Finally, note that whichever of the proposed strategies is used, the $\texttt{broadcast}_{\mathcal{W}_N}$ function increases slightly at the beginning and then decreases.
From this, it can be deduced that at the very
beginning, adding robots to the system
means an increase in the broadcast
time. This means that at first, it is
more difficult to complete a broadcast
when the team is bigger . However,
after a certain threshold, the more
robots the system has, the faster a broadcast
is completed. This behavior can be
expected because after a certain
value of $k$, an increase in the
number of robots helps to spread the
message. 

\subsection{Mixing time}
In this subsection, the mixing time of the distribution corresponding to a drone performing a random walk in an R-SCS is experimentally studied. Our study is based on the discretization of the R-SCS, as introduced in Section~\ref{sec:discretization}, and the transition matrix presented in Section~\ref{sec:transition-theoretical}. The experiment in this section consists of computing the transition matrix $M$ corresponding to an R-SCS $\mathcal{F}$ and then  $M^2, M^3,\dots$ until a matrix close to the uniform distribution is obtained. More precisely, we ask for a matrix $M^t$ such that for any initial distribution $\pi$, 
\begin{equation}\label{eq:error-margin}
\left\Arrowvert \pi
M^t-\pi^*\right\Arrowvert <\varepsilon
\end{equation}
where $ \left\Arrowvert
\cdot\right\Arrowvert$ denotes the
Euclidean norm of a vector, $\pi^*$ is
the uniform distribution (which is the
stationary distribution of the random
walk following the random strategy) and
$\varepsilon$ is a small fixed value.
Then $t$ is an estimate of the mixing
time.

Notice that $\displaystyle\lim_{t\rightarrow\infty}\pi M^t=\pi^*$ for any initial distribution $\pi$, so $\displaystyle\lim_{t\rightarrow\infty}M^t=M^*$, where $M^*$ is the matrix where every entry is $1/n$ if $n$ trajectories are considered. Therefore, $\pi M^* = \pi^*$ and then
\[\left\Arrowvert \pi
M^t-\pi^*\right\Arrowvert
=\left\Arrowvert \pi M^ t-\pi
M^*\right\Arrowvert
= \left\Arrowvert
\pi(M^t-M^*) \right\Arrowvert
\leq\left\Arrowvert
\pi\right\Arrowvert \left\Arrowvert
M^t-M^*\right\Arrowvert,\]
where $ \left\Arrowvert
M^t-M^*\right\Arrowvert$ denotes the
2-norm of the matrix $M^t-M^*$.

Taking into account that
$ \left\Arrowvert
\pi\right\Arrowvert=1$, it can be
deduced that $\left\Arrowvert\pi
M^t- \pi^*\right\Arrowvert\leq
\left\Arrowvert
M^t-M^* \right\Arrowvert$. Then we look
for the minimum value $t$ such that
$\left\Arrowvert
M^t-M^*\right\Arrowvert <\varepsilon$.

For values of $N$ in $\{5,10,15,20,25,30\}$ we build the transition matrix $M$ of an $N\times N$ grid R-SCS and we seek for the minimum value $t$ such that $ \left\Arrowvert
M^t-M^*\right\Arrowvert <\varepsilon=1/
4$. This is a typical value
of $\varepsilon$ for
estimating the mixing time in the litarature~\cite{bayer1992,levin2017markov,shang2018introduction}.

\begin{figure}
    \centering
    \includegraphics[width=\textwidth]{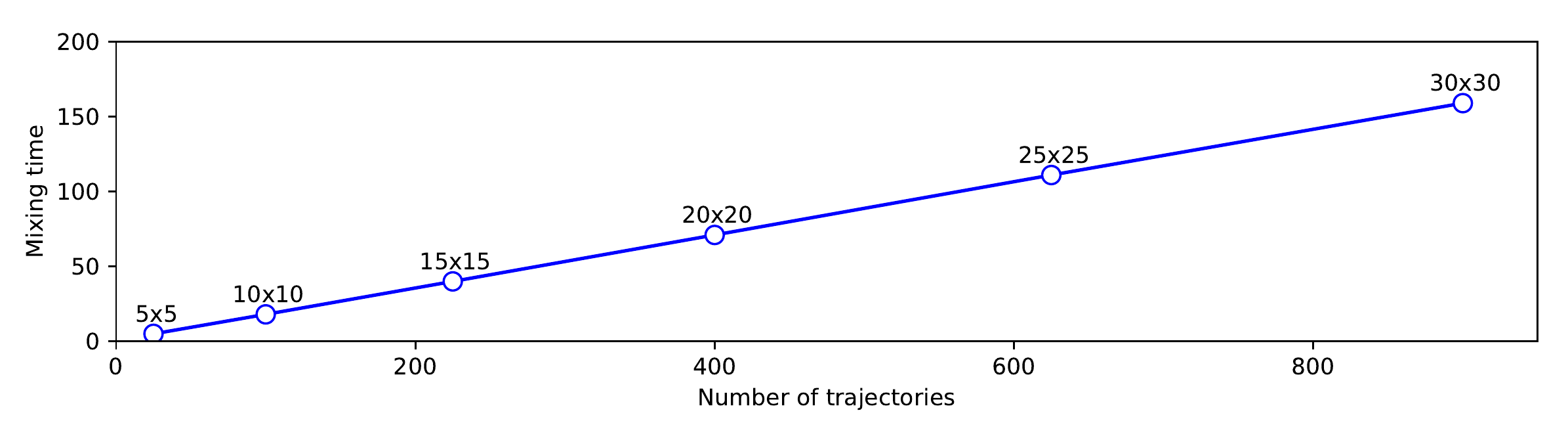}
    \caption{Estimated mixing times in R-SCSs of sizes $5\times 5$, $10\times 10$, $15\times 15$, $20\times 20$, $25\times 25$ and $30\times 30$.}
    \label{fig:mixing_time}
\end{figure}

Figure~\ref{fig:mixing_time} illustrates the results obtained. Note that the behavior of the mixing time seems to be linear with respect to the number of trajectories. The increasing ratio (slope) is approximately 0.17. That is, for every trajectory added to the system, the mixing time increases by 0.17 units. The estimates of mixing times obtained for R-SCSs of sizes $10\times 10$, $15\times 15$, $20\times 20$ and $30\times 30$ (these are the sizes of the R-SCS studied in previous sections) are 18, 40, 71 and 159, respectively.


\subsection{Comparison with the theory}
In this section, the results obtained in the above experiments using the 
random strategy are compared with the theoretical results on idle time and isolation time in Sections \ref{sec:idle} and \ref{sec:isolation} respectively. In addition, a comparison between our experiments on broadcast time using the 
random strategy and the theoretical results on regular graphs~\cite{regular2009} is also included.

\begin{figure}
    \centering
    \begin{subfigure}{.49\textwidth}
    \centering
    \includegraphics[width=\columnwidth]{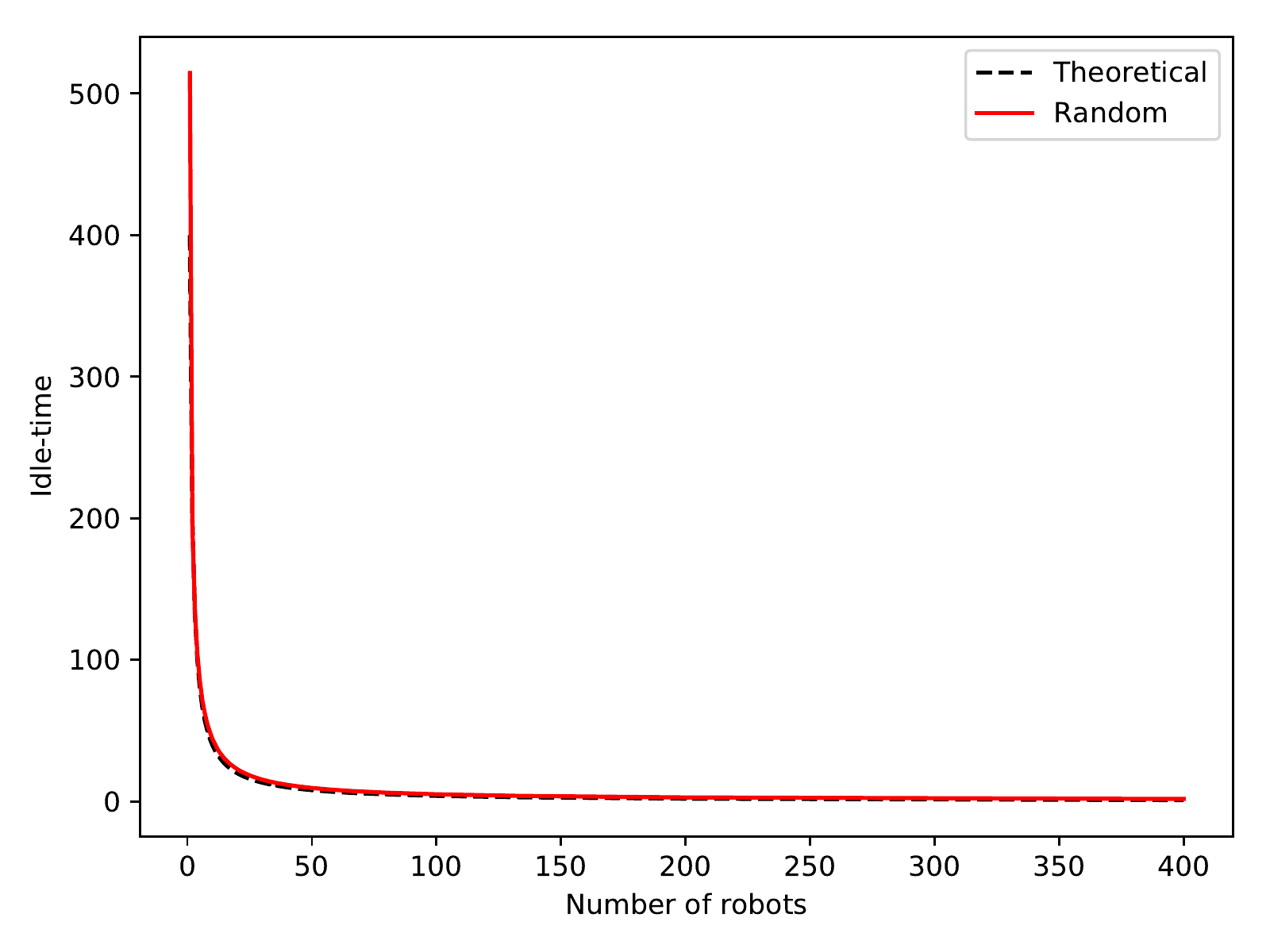}
    \caption{}
    \end{subfigure}
    \begin{subfigure}{.49\textwidth}
    \centering
    \includegraphics[width=\columnwidth]{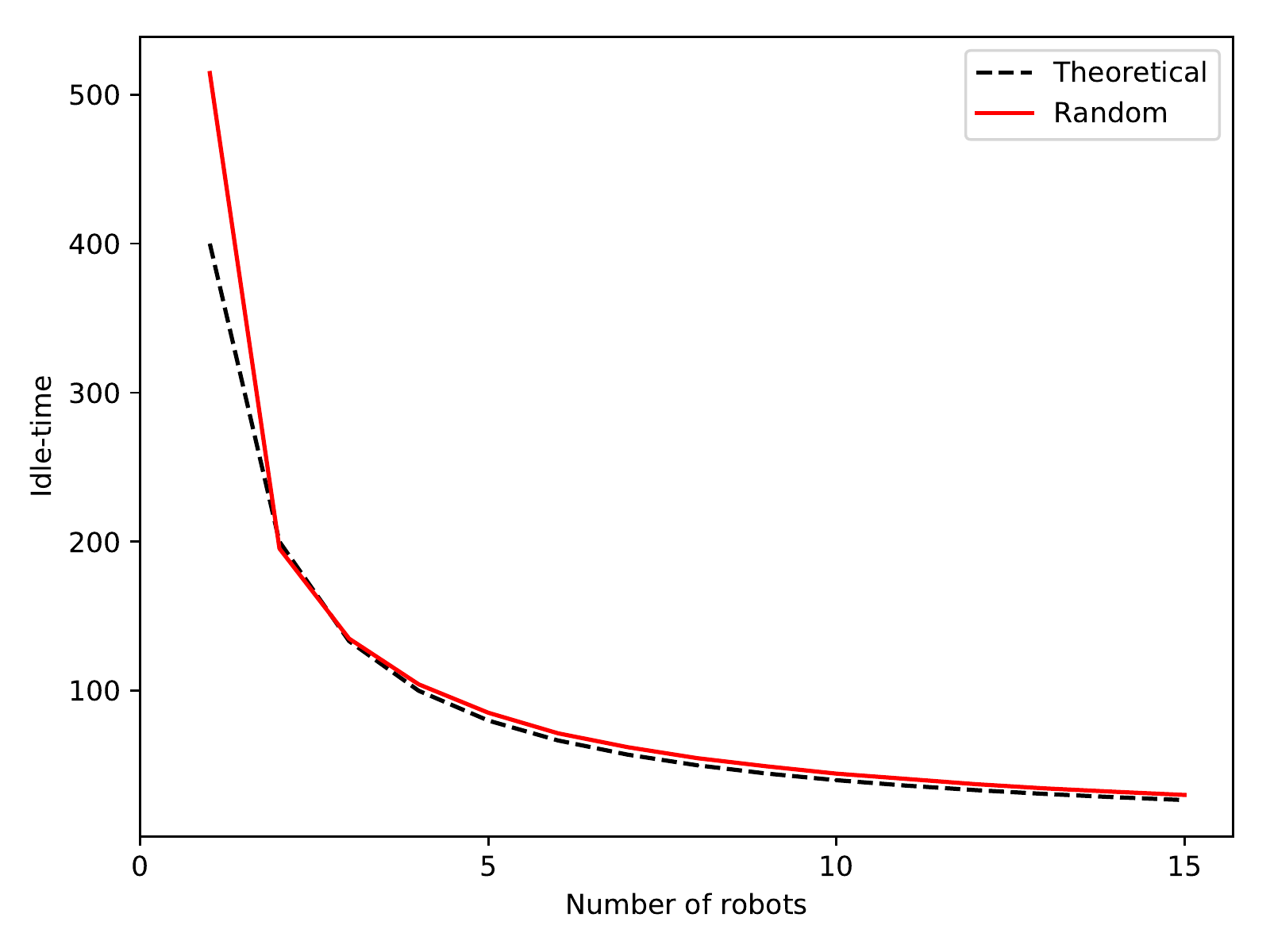}
    \caption{}
    \end{subfigure}
    \caption{Experimental results on idle time and the theoretical times using a $20\times 20$ R-SCS. (a) The behavior of the idle time with respect to $k=1\dots 400$ (number of robots). (b) The same data as (a) but constrained to $k\leq 15$ in order to visualize the differences between the two curves.}
    \label{fig:idle-thx-exp}
\end{figure}

First, it is shown that the bound obtained for idle time is very
tight with respect to the experiments. See Figure~\ref{fig:idle-thx-exp} where the analytical and the experimental results are shown for a $20\times 20$ R-SCS. In Figure~\ref{fig:idle-thx-exp}(a), the results are shown on $k=1\ldots 400$. Note that the analytical and the experimental results almost coincide. Figure~\ref{fig:idle-thx-exp}(b) shows the same data but constrained to $k\leq 15$ in order to show the differences between the two curves better, as they are more evident with fewer robots. However, note that even using a few robots, the 
random strategy has a very robust performance as it is almost optimum (recall that $n/k$ is the best possible idle time in a system of $k$ robots operating in a region divided into $n$ trajectories).

\begin{figure}
    \centering
    \begin{subfigure}{.49\textwidth}
    \centering
    \includegraphics[width=\columnwidth]{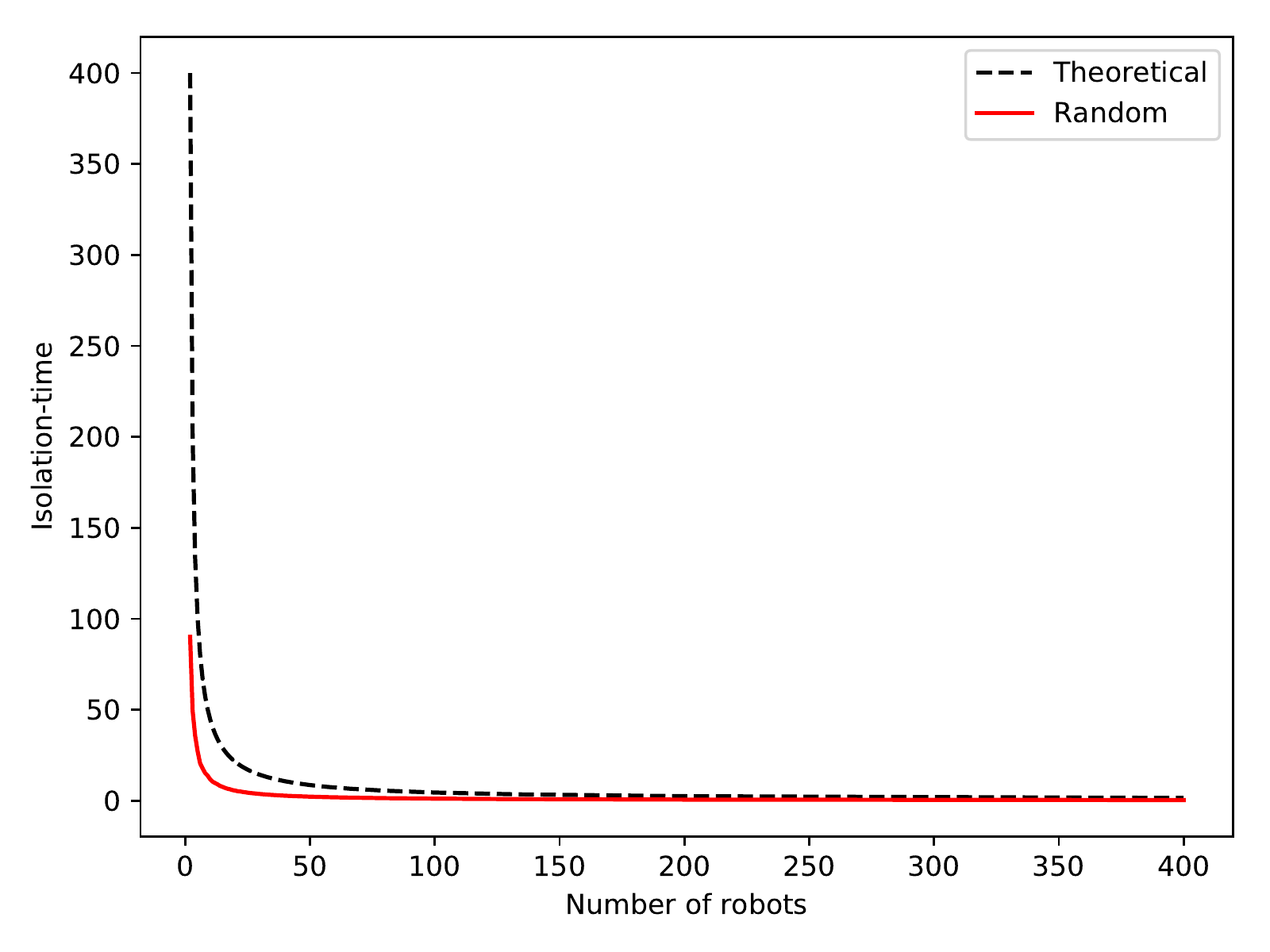}
    \caption{}
    \end{subfigure}
    \begin{subfigure}{.49\textwidth}
    \centering
    \includegraphics[width=\columnwidth]{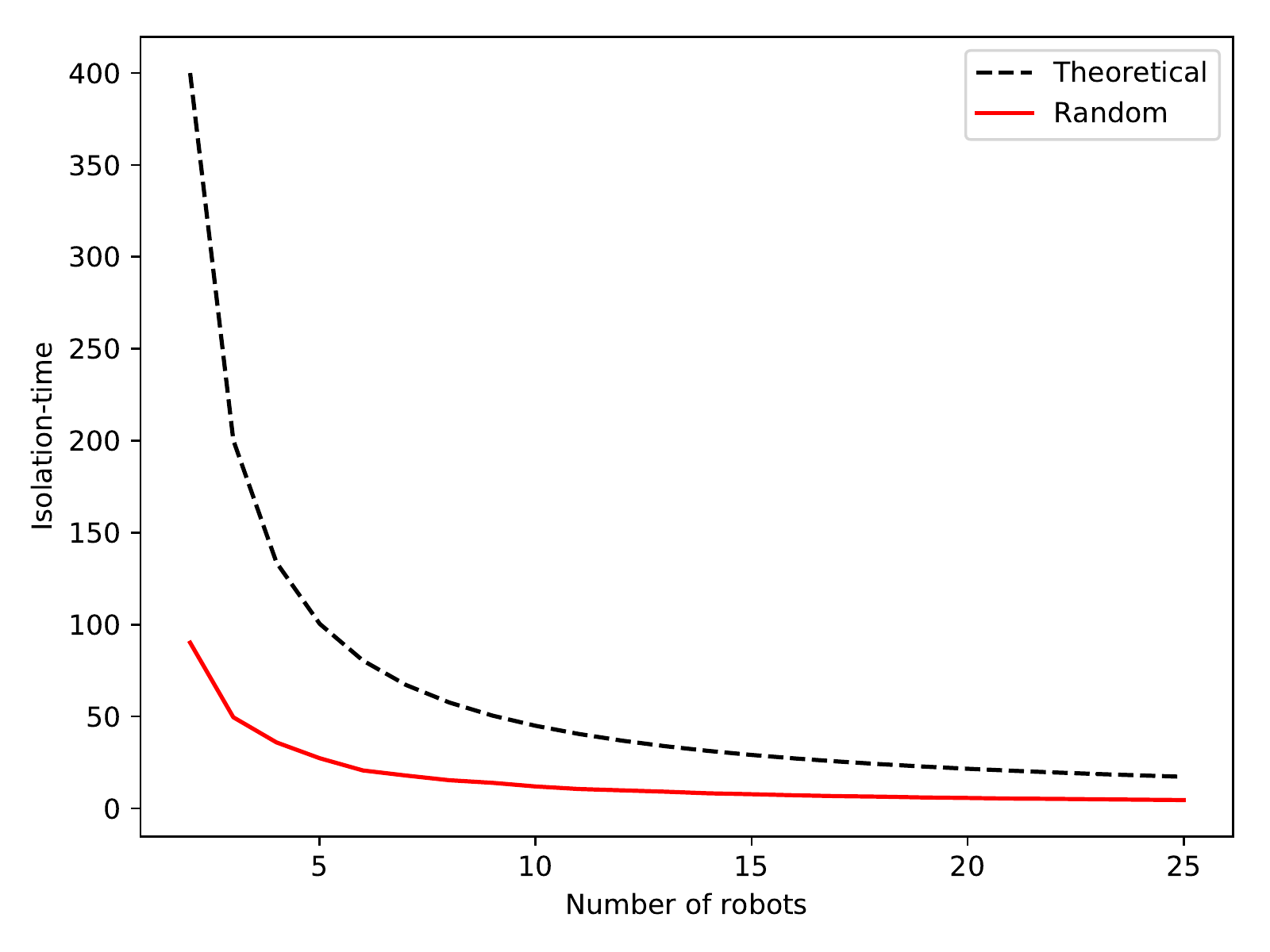}
    \caption{}
    \end{subfigure}
    \caption{Experimental results on isolation time and the theoretical times using a $20\times 20$ R-SCS. (a) The behavior of the isolation time with $k=2\dots 400$ (number of robots). (b) The same data as in (a) but constrained to $k\leq 25$ in order to visualize the differences between the two curves.}
    \label{fig:isolation-thx-exp}
\end{figure}

From the comparison of analytical and experimental isolation times, it can be deduced that the theoretical expected value is an upper bound of the average experimental isolation time, as noted in Section~\ref{sec:isolation} (theoretical study of isolation time), specifically in Remark~\ref{rmk:isolation-upper-bound}. Figure~\ref{fig:isolation-thx-exp} shows the behavior of the theoretical expected isolation time and the average experimental isolation time in $20\times 20$ grid R-SCS. In Figure~\ref{fig:isolation-thx-exp}(a) the global behavior of the two curves is shown for $k=2\dots 400$. Figure~\ref{fig:isolation-thx-exp}(b) shows the same data but constrained to $k\leq 25$ in order to make it easier to visualize the differences between these curves. Note that difference between the theoretical and experimental isolation time decreases as the number of robots increases. This behavior makes sense because the  effects of the phenomenon described in the second
paragraph 
in
Section~\ref{sec:isolation} are more
evident when the number of
trajectories is large with respect to
the number of robots, because there is
a lot of room and it is very probable
that the robots are spread out on the
region. However, if the
number of robots is increased, this effect is
\emph{diluted} because the greater the
number of robots, the greater the
probability of some robot not
traveling alone. 

\begin{figure}
    \centering
    \begin{subfigure}{.49\textwidth}
    \centering
    \includegraphics[width=.7\columnwidth]{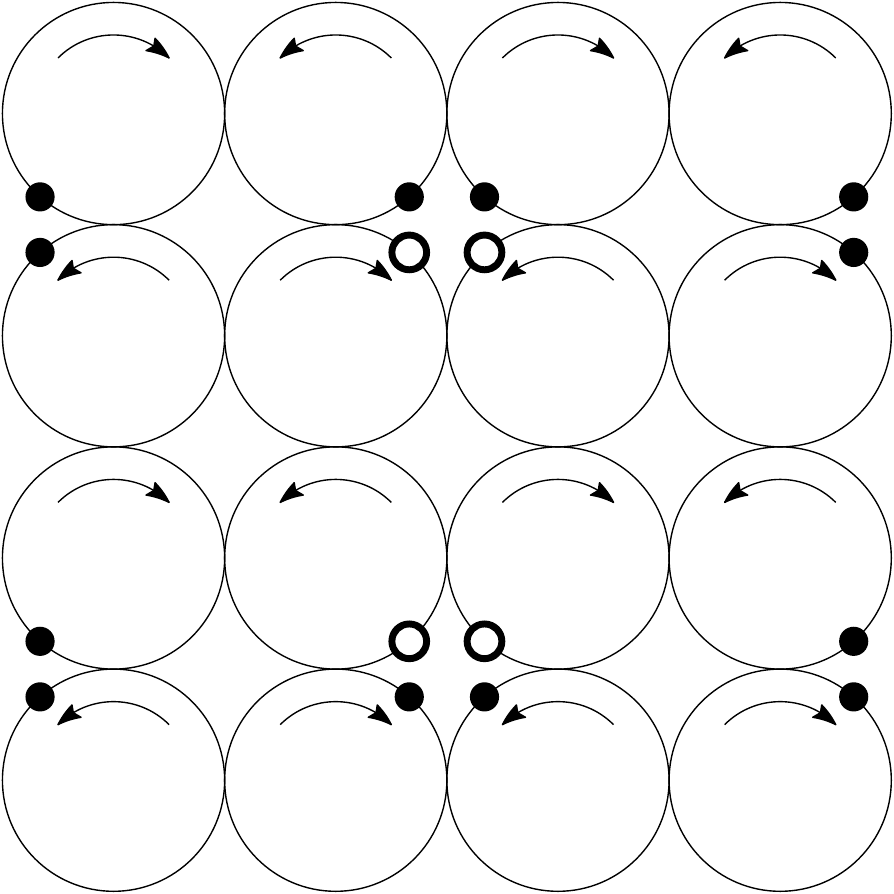}
    \caption{}
\end{subfigure}
     \begin{subfigure}{.49\textwidth}
    \centering
    \includegraphics[width=\columnwidth]{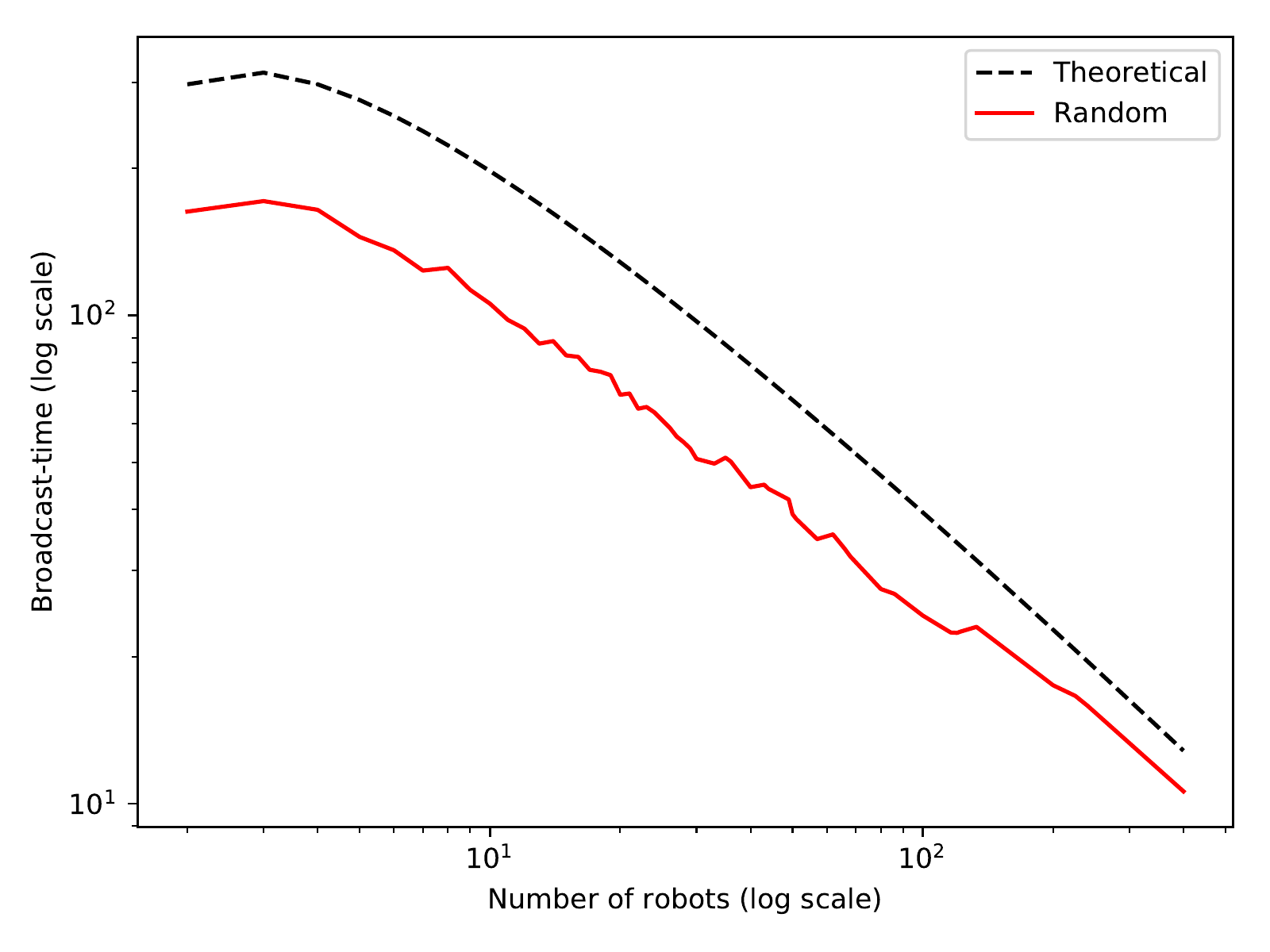}
    \caption{}
    \end{subfigure}
    \begin{subfigure}{.49\textwidth}
    \centering
    \includegraphics[width=\columnwidth]{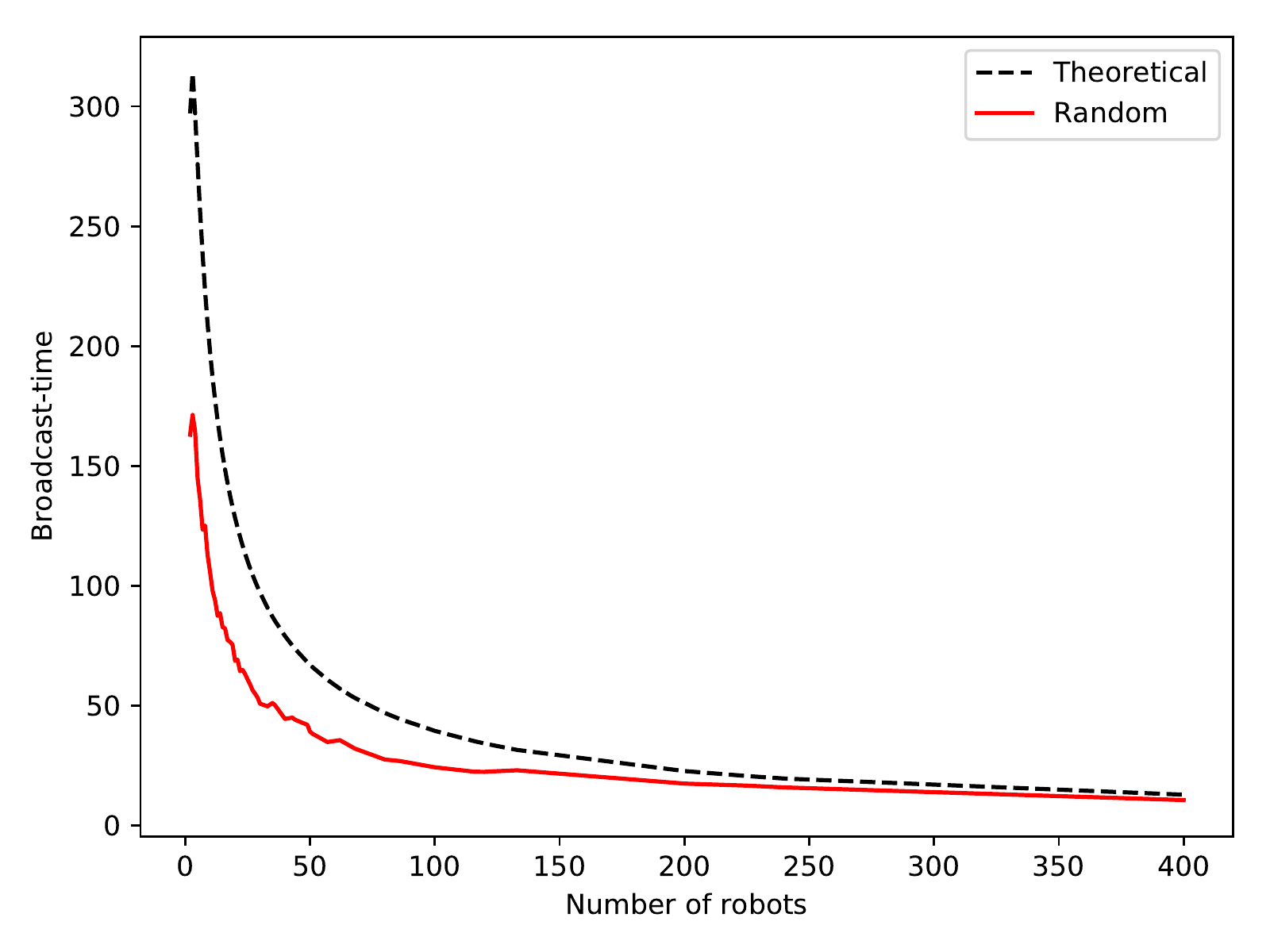}
    \caption{}
    \end{subfigure}
    \begin{subfigure}{.49\textwidth}
    \centering
    \includegraphics[width=\columnwidth]{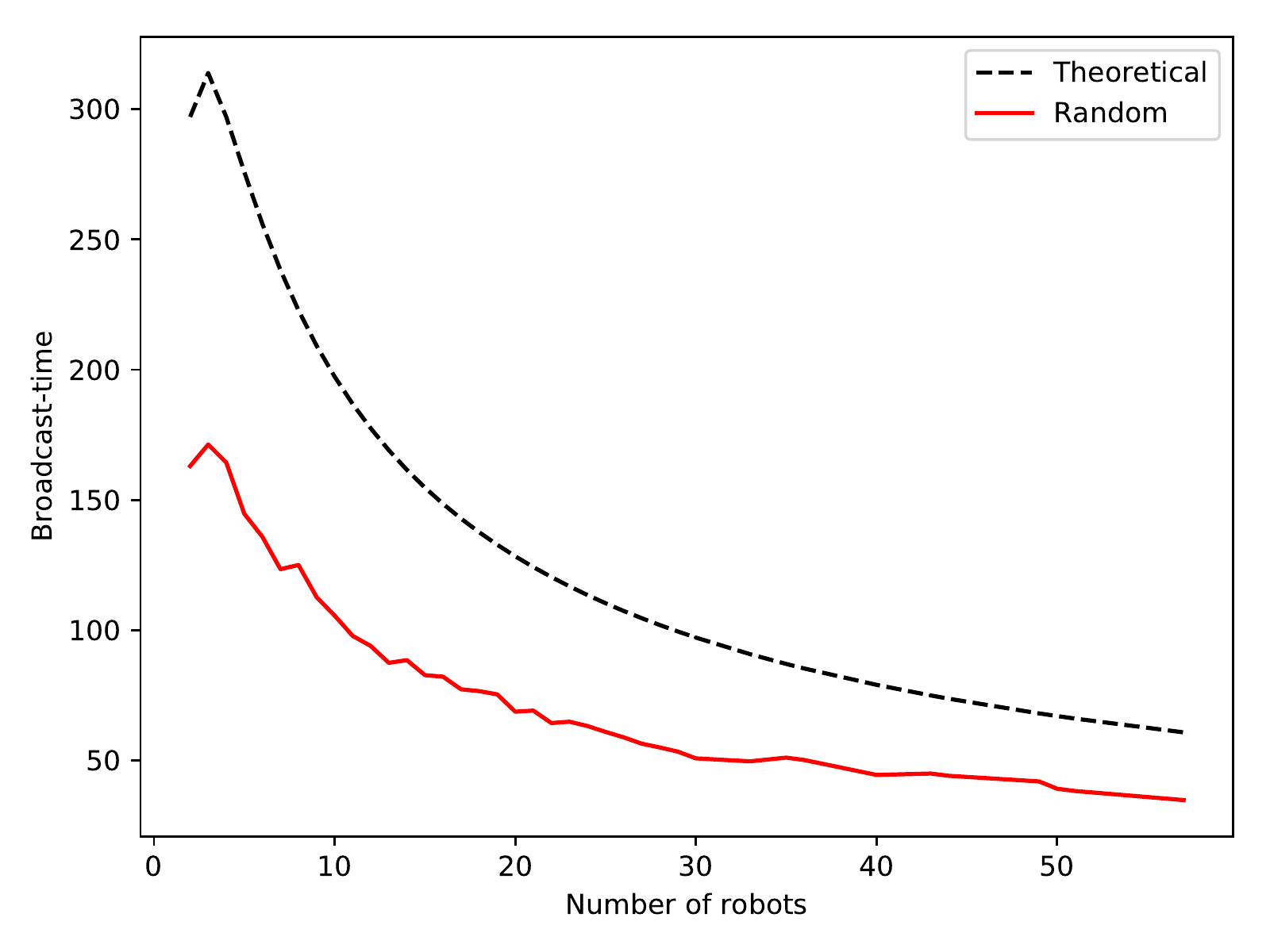}
    \caption{}
    \end{subfigure}
    \caption{(a) A $4\times 4$ section of a grid R-SCS. In (b), (c) and (d) the curve {\it Random} shows the average broadcast time obtained in several experiments in a $20\times 20$ grid R-SCS using $k=2\dots 400$ robots. The curve `Theoretical' shows the behavior of the function $b(k)=\frac{2\ln{k}(\delta-1)n}{k(\delta-2)}$ where $\delta=16$ and $n=400$. (b) Logarithmic scale. (c) Linear scale. (d) The curves constrained to $k\leq 60$.}
    \label{fig:broadcat-time-thx-exp}
\end{figure}

Let us now focus on the broadcast time. Figure~\ref{fig:broadcat-time-thx-exp}(a)
shows a $4\times 4$ section $A$ of a
grid R-SCS $\mathcal{F}$. Let $S$ be
the set of starting positions in $A$.
Let $S'\subset S$ be the set of
starting positions, indicated by open
small circles. Note that from any
point $p'\in S'$ it can be reached a point
$p\in S $ using a path of length
$2\pi$, and all these paths traverse
four communication links. Therefore,
$M_{i,j} = \frac{1}{16}$ for all $i\in
S'$ and $j\in S$, where $M$ is the
transition matrix of the discrete
motion graph corresponding to
$\mathcal{F}$. Observe that if
$\mathcal{F}$ is large, most of the
vertices in the discrete motion graph
have degree 16 and the transition from
these vertices to a neighbor has
probability $\frac{1}{16}$. Therefore,
if $\mathcal{F}$ is large, the
discrete motion graph is almost
regular. Cooper et al.~\cite{regular2009} study the 
broadcast time for regular graphs; 
they state that for most starting
positions, the expected time for $k$ random walks to broadcast a single piece of information to each other is asymptotic to
\[ \frac{2\ln{k }(\delta-1)n}{k (\delta-2
)} \text{\quad as\quad $k ,n\rightarrow
\infty$,} \] where $n$ is the number of vertices and $\delta$ is the degree of the vertices. 

Figures ~\ref{fig:broadcat-time-thx-exp} (b), (c) and (d) show a comparison between the average broadcast time results obtained in the experiments of subsection~\ref{sec:broadcast-exp} using a $20\times 20$ grid R-SCS $\mathcal{F}$ and the curve $b(k)=\frac{2\ln{k}(\delta-1)n}{k(\delta-2)}$, where $n=400$ (number of trajectories on $\mathcal{F}$) and $\delta=16$ (degree of most of the vertices in the discrete motion graph of $\mathcal{F}$). Similarly to the comparison for the isolation time, the curve $b(k)$ is an upper bound of the real broadcast time due to a similar argument. Note that the two curves have very similar shapes. In Figure~\ref{fig:broadcat-time-thx-exp} (d), these curves are shown for $k\leq 60$ in order to make the similarity on their shapes visible. As expected, both curves have a small peak at the beginning, due to the natural behavior of the broadcast time when the number of robots increases (as noted in Section~\ref{sec:broadcast-exp}).

\section{Conclusions}\label{sec:conclusions}

Terrain surveillance using cooperative unmanned aerial vehicles with limited communication range can be modeled as a geometric graph in which the robots share information with their neighbors.  If every pair of neighboring robots  periodically meet at the communication link, the system is called a synchronized communication system (SCS).
In order to add unpredictability to the deterministic framework proposed by D\'iaz-B\'a\~nez et al.~\cite{diaz2017}, the use of stochastic strategies in the SCS is studied. 
Both
the coverage and communication performance are evaluated and the validity of two random strategies compared with the deterministic one has been proved.  The
performance metrics of interest focused on in this paper were
\emph{idle time} (the expected time between two consecutive observations of any point of the system), \emph{isolation time} (the expected time that a robot is isolated) and \emph{broadcast time} (the expected time elapsed from the moment a robot emits a message until it is received by all the other robots in the team). 

First, theoretical results have been proved for one of the strategies, called random strategy, which can be modelled by classical random walks, and obtained bounds assuming that the starting vertices are uniformly selected. A theoretical study of the other protocol, the quasi-random strategy, which does not generate random walks, is a challenging open problem. 
After that, three computational studies were carried out: a comparison between the deterministic and the random strategies; a simulation for estimating the mixing time of the system (roughly, the time needed for a random walk to reach its stationary distribution); and a comparison between the simulation and the theoretical bounds.

We summarize our observations in the experiments. 
Overall, the behavior was analyzed by increasing the size of the team of UAVs working on grid graphs.
For the first study, our results showed that if the system has few robots, as it is usual in practice, it is better to use one of the random strategies than the deterministic one. Indeed, in this case, the results obtained with random strategies were much better than the behavior obtained using the deterministic strategy for the three tested measures.

Although clever initial positions for the robots can selected in the deterministic strategy such that the properties of communication and coverage are satisfied, the system does not take unpredictability into account. If a group of robots fail, the system can lose robustness using the deterministic strategy. As the experiments showed, this drawback can be overcome using random strategies.
The experiments also showed that the behavior of the two random strategies is very similar for any of the proposed quality measures. 
In the second study, it was observed that the estimated mixing time is small compared to the number of trajectories in the system and it only increases 0.17 units of times per trajectory.
The third experimental study showed that the theoretical results are tight for the idle time and give reasonable upper bounds for the isolation time and broadcast time.

Finally, it should be noted that more general SCSs in which the trajectories are closed curves with different lengths could be established \cite{diaz2017} and the results in this paper are extensible to those SCSs. Future research could also focus on considering other random strategies, different topologies in the experiments, as well as a study on the influence of the value of the probability $p$, fixed as $0.5$ in this paper.


\bibliographystyle{abbrv} \bibliography{random_drones}
\end{document}